\documentclass{new_tlp}

\usepackage[utf8]{inputenc}
\usepackage{amsmath,amsfonts,amssymb}
\usepackage{microtype}
\usepackage{yhmath}
\usepackage{upgreek}
\usepackage{url}

\usepackage{stmaryrd}
\usepackage{listings}
\newcommand{\head}[1]{\ensuremath{\mathit{h}(#1)}} \newcommand{\Head}[1]{\ensuremath{\mathit{H}(#1)}}
\newcommand{\body}[1]{\ensuremath{\mathit{B}(#1)}}

\newcommand{\poslits}[1]{\ensuremath{{#1}^+}}
\newcommand{\neglits}[1]{\ensuremath{{#1}^-}}

\newcommand{\pbody}[1]{\poslits{\body{#1}}}
\newcommand{\nbody}[1]{\neglits{\body{#1}}}

  \newcommand{\HT}{\ensuremath{\mathrm{HT}}}

\newcommand{\HTC}{\ensuremath{\mathrm{HT}_c}}

\newcommand{\reduct}[2]{\ensuremath{#1^{#2}}}

\newcommand{\tuple}[1]{\langle #1 \rangle}

\providecommand{\Underscore}{\textunderscore}

\lstdefinelanguage{clingo}{basicstyle=\ttfamily,keywordstyle=[1]\bfseries,keywordstyle=[2]\bfseries,keywordstyle=[3]\bfseries,showstringspaces=false,literate={_}{\Underscore}1 {\%\%}{}0,escapeinside={\#(}{\#)},alsoletter={\#,\&},keywords=[1]{not,from,import,def,if,else,elif,return,while,break,and,or,for,in,del,and,class,with,as,is,yield,async},keywords=[2]{\#const,\#show,\#minimize,\#base,\#theory,\#count,\#external,\#program,\#script,\#end,\#heuristic,\#edge,\#project,\#show,\#sum},keywords=[3]{&,&dom,&sum,&diff,&show},morecomment=[l]{\#\ },morecomment=[l]{\%\ },morestring=[b]",stringstyle={\itshape},commentstyle={\color{darkgray}}}

\lstdefinelanguage{python}{basicstyle=\ttfamily,keywordstyle=[1]\bfseries,showstringspaces=false,literate={_}{\Underscore}{1},escapeinside={\#(}{\#)},alsoletter={\#,\&},keywords=[1]{not,from,import,def,if,else,elif,return,while,break,and,or,for,in,del,and,class,with,as,is,yield,async},morecomment=[l]{\#\ },morestring=[b]",stringstyle={\itshape},commentstyle={\color{darkgray}}}

 \newcommand{\sysfont}{\textit}

\newcommand{\clingcon}{\sysfont{clingcon}}
\newcommand{\clingo}{\sysfont{clingo}}

\newcommand{\dlvhex}{\sysfont{dlvhex}}

\newcommand{\clingoM}[1]{\clingo{\small\textnormal{[}\textsc{#1}\textnormal{]}}}

\newcommand{\C}{C}

 \renewcommand{\sysfont}{\texttt}

\newtheorem{proposition}{Proposition}
\newtheorem{theorem}{Theorem}

\newtheorem{lemma}{Lemma}

\newcommand{\den}[1]{\llbracket \, #1 \, \rrbracket}
\newcommand{\denn}[1]{\den{#1}}
\newcommand{\comp}[1]{\wideparen{#1}}

\newcommand{\tprogram}{$\mathcal{T}$-logic program}
\newcommand{\tsolution}{$\langle\AT,\TheoryAtomsES\rangle$-solution}
\newcommand{\lsolution}{$\langle\LC,\TheoryAtomsES\rangle$-solution}
\newcommand{\conprogram}{\clingcon-program}

\newcommand{\Atoms}{\ensuremath{\mathcal{A}}}
\newcommand{\TheoryAtoms}{\ensuremath{\mathcal{T}}}
\newcommand{\TheoryAtomsDN}{\ensuremath{\mathcal{F}}}
\newcommand{\TheoryAtomsES}{\ensuremath{\mathcal{E}}}
\newcommand{\htcsemanticsP}{\ensuremath{\uptau}}
\newcommand{\htcsemanticsA}{\ensuremath{\uptau}}
\newcommand{\htctsols}{\ensuremath{\Phi}}
\newcommand{\htctrans}{\ensuremath{\Gamma}}

\newcommand{\X}{\ensuremath{\mathcal{X}}}
\newcommand{\V}{\ensuremath{\mathcal{V}}}
\newcommand{\D}{\ensuremath{\mathcal{D}}}
\newcommand{\Du}{\ensuremath{\mathcal{D}_{\undefined}}}
\renewcommand{\C}{\ensuremath{\mathcal{C}}}

\newcommand{\AT}{\ensuremath{\mathfrak{T}}}
\newcommand{\LC}{\ensuremath{\mathfrak{L}}}
\newcommand{\DLAT}{\ensuremath{\mathfrak{D}}}
\newcommand{\LP}{\ensuremath{\mathfrak{R}}}

\newcommand{\vars}[1]{\ensuremath{\mathit{vars}(#1)}}
\newcommand{\varsTfunction}[1]{\ensuremath{\mathit{vars}_{#1}}}
\newcommand{\varsT}[2]{\ensuremath{\varsTfunction{#1}(#2)}}
\newcommand{\varsATfunction}{\varsTfunction{\AT}} \newcommand{\varsAT}[1]{\varsT{\AT}{#1}}

\newcommand{\undefined}{\ensuremath{\mathbf{u}}} 
\newcommand{\dfZ}[1]{\mathit{def}_{\mathbb{Z}}(#1)}
\newcommand{\eqdef}{\mathrel{\vbox{\offinterlineskip\ialign{\hfil##\hfil\cr $\scriptscriptstyle\mathrm{def}$\cr \noalign{\kern1pt}$=$\cr \noalign{\kern-0.1pt}}}}}
\newcommand{\restr}[2]{{#1|}_{\hspace{-1pt}#2}}

\newcommand{\code}[1]{\ensuremath{\mathtt{#1}}} \newcommand{\entails}{\mathbin{|\kern -.42em\approx}}
\newcommand{\nentails}{\mathbin{|\kern -.47em\approx\kern-.9em{/}}}

\def\sumxy{\code{\&sum\{x;y\}=4}}
\def\sumxyc{\code{\&sum\{x;y\}!=4}}
\def\sumyz{\code{\&sum\{y;z\}=2}}
\def\sumyzc{\code{\&sum\{y;z\}!=2}}
\newcommand{\AMT}{AMT}
\newcommand{\prop}[1]{p_{#1}}
\def\bridge{\mathit{Bridge}}
\def\anyAt{b}
\def\regAt{\code{a}}
\def\thAt{\code{s}}
\def\prAt{\prop{}}
\newcommand\diffc[3]{\code{\&diff\{\mathit{#1}-\mathit{#2}\} \ \mbox{\tt <=} \ \mathit{#3}}}

\newcommand{\qed}{\hspace*{1em}\hbox{\proofbox}}
\newcommand{\ASPAC}{\ensuremath{\mathit{ASP}({\mathcal{AC}})}}

\allowdisplaybreaks

\begin{document}

\title{Towards a Semantics for Hybrid ASP systems}

\author[Pedro Cabalar et al]{Pedro Cabalar$^1$ \ Jorge Fandinno$^{2,3}$ \ Torsten Schaub$^3$ \ Philipp Wanko$^3$
  \\$^1${University of Corunna, Spain} \ $^2${Omaha State University, USA} \ $^3${University of Potsdam, Germany} }

\maketitle

\begin{abstract}
Over the last decades the development of ASP has brought about an expressive modeling language powered by highly performant systems.
At the same time, it gets more and more difficult to provide semantic underpinnings capturing the resulting constructs and inferences.
This is even more severe when it comes to hybrid ASP languages and systems that are often needed to handle real-world applications.
We address this challenge and introduce the concept of abstract and structured theories that allow us to
formally elaborate upon their integration with ASP.
We then use this concept to make precise the semantic characterization of \clingo's theory-reasoning framework
and establish its correspondence to the logic of Here-and-there with constraints.
This provides us with a formal framework in which we can elaborate formal properties of existing
hybridizations of \clingo\ such as \clingcon, \clingoM{dl}, and \clingoM{lp}.
\end{abstract}
 \section{Introduction}\label{sec:introduction}

Answer Set Programming (ASP) is about mapping a logic program onto its so-called stable models.
Over the decades, stable models have been characterized in various ways~\cite{lifschitz10a},
somehow underpinning their appropriateness as a semantics for logic programs.
However, over the same period, the syntax of logic programs has been continuously enriched,
equipping ASP with a highly attractive modeling language.
This development also brought about much more intricate semantics,
as nicely reflected by the mathematical apparatus used in the original definition~\cite{gellif88b} and the one for
capturing the input language of the ASP system \clingo~\cite{gehakalisc15a}.

With the growing range of ASP applications in academia and industry~\cite{fafrsctate18a}, we also witness the emergence of more and more hybrid ASP systems~\cite{lierler14a,kascwa17a},
similar to the raise of SMT solving from plain SAT solving~\cite{baseseti09a}.
From a general perspective, the resulting paradigm of ASP modulo theories (\AMT) can be seen as a refinement of ASP,
in which an external theory certifies some of a program's stable models.
A common approach, stemming from lazy theory solving~\cite{baseseti09a}, is to view a theory as an oracle that
determines the truth of a designated set of (theory) atoms in the program.
This idea has meanwhile been adapted to the setting of ASP in various ways, most prominently in
Constraint ASP~\cite{lierler14a} extending ASP with linear equations. Beyond this, ASP systems like \clingo\ or \dlvhex~\cite{eigeiakarescwe18a} leave the choice of the specific
theory largely open.
For instance,
\clingo\ merely requires fixing the interaction between theories and their theory atoms. As attractive as this generality may be from an implementation point of view,
it complicates the development of generic semantics that are meaningful to existing systems.
Not to mention that in ASP the integration of theories takes place in a non-monotonic context.

We address this issue and show how the Logic of Here-and-There with constraints (\HTC;~\citeNP{cakaossc16a})
can be used as a semantics for \clingo's theory reasoning framework.
Thus, just like the plain Logic of Here-and-There (\HT;~\cite{heyting30a,pearce96a}) serves as the logic foundations of ASP,
\HTC\ extends this to \AMT.
In this way, we cannot only draw on the formal framework of \HT\ but we can moreover study a heterogeneous approach such as
\AMT\ in a the uniform setting of a single logic.
To this end,
we introduce the concept of abstract and structured theories that allow us to formally elaborate upon their integration with ASP.
With them, we make precise the semantic characterization of \clingo's theory-reasoning framework
and establish its correspondence to theories in \HTC.
This provides us with a formal framework in which we can elaborate formal properties of existing
hybridizations of \clingo.

 \section{Background}\label{sec:background}

We consider an alphabet consisting of two disjoint sets, namely,
a set \Atoms\ of propositional (regular) atoms and a set \TheoryAtoms\ of theory atoms, whose truth is governed by some external theory.
Several theory atoms may represent the same theory entity.
For example,
\clingcon\ extends the input language of \clingo\ with linear equations over integers.
Each such linear equation is represented by a theory atom of form
\begin{align}\label{clingcon:linear:constraint}
  \code{\&sum\{\mathit{k_1*x_1};\dots;\mathit{k_n*x_n}\}} \prec k_0
\end{align}
where $x_i$ is an integer variable and $k_i\in\mathbb{Z}$ an integer constant for $0\leq i\leq n$;
and $\prec$ is a comparison symbol such as \code{<=}, \code{=}, \code{!=}, \code{<}, \code{>}, \code{>=}.
In \clingo, theory predicates are preceded by `\code{\&}'.
We use letters $\regAt$, $\thAt$, and $\anyAt$ for atoms in $\Atoms$, $\TheoryAtoms,$ and $\Atoms \cup \TheoryAtoms$, respectively.

A \tprogram\ over $\langle\Atoms,\TheoryAtoms\rangle$ is a set of rules of the form
\begin{align}
    \anyAt_0 \leftarrow \anyAt_1,\dots,\anyAt_n,\neg \anyAt_{n+1},\dots,\neg \anyAt_m \label{theory:rule}
\end{align}
where $\anyAt_i\in\mathcal{A}\cup \TheoryAtoms$ for $0 \leq i \leq m$ and
$\anyAt_0$ can also be the falsity constant $\bot$.
We let $\head{r}=b_0$ and $\body{r}=\{\anyAt_1,\dots,\anyAt_n,\neg \anyAt_{n+1},\dots,\neg \anyAt_m\}$ stand for
the \emph{head} and \emph{body} of a rule $r$ in~\eqref{theory:rule};
its \emph{positive} and \emph{negative body atoms} are given by
$\pbody{r}=\{\anyAt_1,\dots,\anyAt_n\}$ and $\nbody{r}=\{\anyAt_{n+1},\dots,\anyAt_m\}$, respectively.
For a program $P$, we define $\Head{P}=\{\head{r}\mid r\in P\}$,
$\body{P}=\{\body{r}\mid r\in P\}$.
A system using \tprogram{s} is, for instance,
\clingcon\ where theory atoms of form \eqref{clingcon:linear:constraint} correspond to linear equations over integers.
We refer to logic programs augmented by such theory atoms as \conprogram{s}.
 Next, we describe the semantics of \tprogram{s} in \clingo~\cite{gekakaosscwa16a},
used for extending it with difference constraints and linear equations over integers and reals
in \clingoM{dl}, \clingcon, and \clingoM{lp}, respectively.
In all these systems, we can see the role of the corresponding external theory as a kind of ``\emph{certification authority}''
that sanctions regular stable models whose theory atoms are in accord with their underlying constraints.
To this end, envisage a two-step process:
(1) generate regular stable models and
(2) select the ones passing the theory certification.
In step~(1),
we ignore the distinction between atoms in \Atoms\ and \TheoryAtoms\ and
plainly apply the stable model semantics~\cite{gellif88b}:
A set $X\subseteq{\Atoms\cup\TheoryAtoms}$ of atoms is a \emph{model} of a
\tprogram~$P$ over $\langle\Atoms,\TheoryAtoms\rangle$,
if
$\head{r}\in X$
whenever
$\pbody{r}\subseteq X$ and $\nbody{r}\cap X=\emptyset$
for all $r\in P$.
The stable models of~$P$ are defined via the \emph{reduct} of $P$ relative to a set~$X$ of atoms, viz.\
\(
\reduct{P}{X}
=
\{\head{r}\leftarrow\pbody{r}\mid r\in P, \nbody{r}\cap X=\emptyset\}
\).
Then,
a set $X$ of atoms is a \emph{stable model} of~$P$,
if $X$ is the least model of~\reduct{P}{X}.
In~(2),
an external theory \AT\ is used to eliminate any stable model $X$ whose theory atoms $X \cap \TheoryAtoms$ are not in \emph{accord} with \AT.
Suppose step~(1) yields a stable model of a \conprogram\ containing both atoms
$\code{\&sum\{\mathit{2*x;4*y}\}<=\mathit{6}}$ and $\code{\&sum\{\mathit{x;2*y}\}>\mathit{9}}$.
This stable model would not pass theory certification, since no values for integer variables $x$ and $y$ can satisfy both linear equations.

In fact, the treatment of each theory atom $\thAt$ (as obtained in the stable model) leaves room for different semantic options~\cite{gekakaosscwa16a}.
A first option is about the justification for the truth of $\thAt$ in step~(1).
If $\thAt$ occurs in a rule head, it is called \emph{defined} and its truth must be derived through the program rules.
Otherwise, $\thAt$ is said to be \emph{external} and we are always free to add it to the program without further justification.
A second semantic option has to do with how we treat $\thAt \not\in X$ for a stable model $X$ in step~(2).
We say that $\thAt$ is \emph{strict} when $\thAt \not\in X$ implies that the theory check requires the ``opposite'' constraint of $\thAt$ to be satisfied.
Otherwise, we say that $\thAt$ is \emph{non-strict} and the fact $\thAt \not\in X$ has no relevant effect during theory certification.
Of course, in both cases $\thAt \in X$ implies that the constraint of $\thAt$ is satisfied.
As an example, suppose $\thAt$ is the strict atom \code{\&sum\{\mathit{2*x;4*y}\}<=\mathit{6}} and we get $\thAt \not\in X$ in some stable model $X$.
Then, step~(2) imposes the same constraint as if we had obtained atom $\code{\&sum\{\mathit{2*x;4*y}\}>\mathit{6}}$ in $X$.

\citeN{jakaosscscwa17a} argue that only the combinations (non-strict,defined) and (strict,external) are sensible
for linear equation theories. With our focus on these theories, we follow this proposition
and partition the set \TheoryAtoms\ of theory atoms into
two disjoint sets, namely,
\TheoryAtomsDN, containing theory atoms interpreted in a (non-strict,defined) manner, and
\TheoryAtomsES, including the ones abiding by a (strict,external) interpretation.
We refer to theory atoms in \TheoryAtomsDN\ and \TheoryAtomsES\ as \emph{founded} and \emph{external}, respectively.
Intuitively, atoms in \TheoryAtomsDN\ are derived by the \tprogram,
while the ones in \TheoryAtomsES\ are determined by the theory.
Accordingly, for any \tprogram\ $P$ over \Atoms\ and \TheoryAtoms\ with partition
$\langle\TheoryAtomsDN,\TheoryAtomsES\rangle$, we require  that
$\Head{P}\cap\TheoryAtomsES=\emptyset$ and $\body{P}\cap\TheoryAtomsDN=\emptyset$.
We refer to such programs to be over $\langle\Atoms,\TheoryAtoms,\TheoryAtomsES\rangle$ by leaving
\TheoryAtomsDN\ implicit.

With this restriction to two interpretations of theory atoms,
we can formulate \clingo's semantics for \tprogram{s} as follows~\cite{gekakaosscwa16a}:
Informally speaking, a \AT-\nolinebreak{}solution $S$ is a subset $S\subseteq\TheoryAtoms$ of theory atoms
whose associated constraints are ``sanctioned'' by \AT; this is made precise in Section~\ref{sec:translation:revisited}.
A set $X\subseteq \Atoms\cup\TheoryAtoms$ of atoms is a $\AT$-stable model of a \tprogram\ $P$ over $\langle\Atoms,\TheoryAtoms,\TheoryAtomsES\rangle$,
if there is some \AT-solution $S$ such that $X$ is a stable model of the logic program
\begin{align}\label{eq:transformation}
  P
  \cup
  \{{\thAt\leftarrow} \mid \thAt\in (S\cap\TheoryAtomsES)\}
  \cup
  \{ {\leftarrow \thAt} \mid \thAt\in \TheoryAtoms \setminus (S \cup \TheoryAtomsES)\}
\end{align}

For illustration,
consider the \conprogram~$P_{(\ref{ex:clingcon:rule:one}/\ref{ex:clingcon:rule:two})}$:
\begin{align}
  \label{ex:clingcon:rule:one}
  \code{a}                   &\ \leftarrow \ \code{\&sum\{ x; y \} = 4}\\
  \label{ex:clingcon:rule:two}
  \code{\&sum\{ y; z \} = 2} &\ \leftarrow \ \code{a}
\end{align}
This program contains two theory atoms: $\thAt_1$=(\code{\&sum\{x;y\}=4}) is external (it occurs in a rule body) whereas $\thAt_2$=(\code{\&sum\{y;z\}=2}) is founded (it only occurs in rule heads).
To obtain the stable models of $P_{(\ref{ex:clingcon:rule:one}/\ref{ex:clingcon:rule:two})}$, we must decide first a potential set $S$ of theory atoms.
First, note that the shape of such \AT-solutions $S$, and ultimately \AT-stable models,
depends on the theory and external atoms in \TheoryAtoms\ and \TheoryAtomsES.
For the theory of linear equations, the intention of external atoms is to find an assignment satisfying the represented constraint or its complement
whenever the external atom is true or false, respectively.
Therefore, the existence of $\thAt_1$ implicitly requires another external theory atom $\thAt_3=\code{\&sum\{x;y\}!=4}$.
Note that in systems like \clingcon, answers neither contain $\thAt_1,\thAt_2$ nor $\thAt_3$ explicitly.
For a \AT-stable model $X$, the answer of the system consists of $X\cap\Atoms$,
plus an assignment (or ``\emph{witnesses}'') satisfying linear equations represented by $X\cap\TheoryAtoms$.
Then, we are free to either add the external atom $\thAt_1$ or $\thAt_3$ as a fact, whereas if we decide to leave $\thAt_2 \not\in S$, we must add the constraint $\leftarrow \thAt_2$.
This leads to two possibilities: $X_1=\{\thAt_1,\thAt_2,\code{a}\}$ and $X_2=\{\thAt_3\}$.
If \AT\ is the theory of linear equations, there exist several \AT-solutions $S$ containing $\{\thAt_1,\thAt_2\}$ that justify $X_1$ since there exist multiple integer values for $x,y,z$ satisfying equations $x+y=4$ and $y+z=2$ (e.g.\ $x=2, y=2, z=0$).
In the case of $X_2$, since the external atom $\thAt_3 \in X_2$, any \AT-solution should satisfy $x+y\neq 4$.
But again, we may easily find sets $S$ of linear equations that are in accord with that constraint.

 The logic of \emph{Here-and-there with constraints} (\HTC;~\citeNP{cakaossc16a}) is an extension of Equilibrium
Logic~\cite{pearce96a} providing logical foundations for constraint satisfaction problems (CSPs) in the setting of ASP.
In \HTC, a CSP is expressed as a triple $\tuple{\X,\D,\C}$ (we also call \emph{signature}),
where \X\ is a set of \emph{variables} over some non-empty \emph{domain} \D,
and \C\ is a set of \emph{constraint atoms}.
A constraint atom provides an abstract way to relate values of variables and constants
according to the atom's semantics.
Most useful constraint atoms have a structured syntax, but in the general case, we may simply consider them as strings.
For instance, linear equations are expressions of the form ``$x+y=4$'',
where $x$ and $y$ are variables from~\X\ and $4$ is a constant representing some element from \D.

Variables can be assigned some value from $\D$ or left \emph{undefined}.
For the latter, we use the special symbol $\undefined \notin \D$ and the extended domain $\Du \eqdef \D \cup \{\undefined\}$.
The set $\mathit{vars}(c) \subseteq \X$
collects all variables occurring in constraint atom $c$.
We assume that every constraint atom~$c$ satisfies~$\vars{c} \neq \emptyset$
(otherwise it is just a truth constant).
A \emph{valuation} $v$ over $\X,\D$ is a function $v:\X\rightarrow\Du$.
Moreover, valuation $v|_X: X\rightarrow\Du $ is the projection of $v$ on $X \subseteq \X$.
A valuation $v$ can be represented as the set
\(
\{ (x,v(x)) \mid x \in \X, v(x)\in\D\}
\),
thus excluding pairs of form $(x,\undefined)$.
Hence, $v\subseteq v'$ stands for
\(
\{ (x,v (x)) \mid x \in \X, v (x)\in\D\}
\subseteq
\{ (x,v'(x)) \mid x \in \X, v'(x)\in\D\}
\).
We also allow for applying valuations $v$ to fixed values,
and so extend their type to
\(
v:\X\cup \Du\rightarrow\Du
\)
by fixing $v(d) = d$ for any $d \in \D_\undefined$.
We let $\mathcal{V}_{\X,\D}$ (or simply $\mathcal{V}$) stand for
the set of all valuations over $\X,\D$.

The semantics of constraint atoms is defined in \HTC\ via \emph{denotations},
which are functions
\(
\den{\cdot}:\C\rightarrow 2^{\mathcal{V}}
\),
mapping each constraint atom to a set of valuations.
They must satisfy the following properties
for all
$c\in\C$,
$x\in\X$, and
$v,v' \in \mathcal{V}$~\cite{cafascwa20a}:
\begin{enumerate}
\item $v \in \den{c}$ and $v \subseteq v'$ imply $v' \in \den{c}$,
  \label{den:prt:0}
\item $v \in \den{c}$ implies $v \in \den{c[x/v(x)]}$,
  \label{den:prt:1}
\item if $v(x)=v'(x)$ for all $x \in \mathit{vars}(c)$ then $v \in \den{c}$ iff $v' \in \den{c}$.
  \label{den:prt:2}
\end{enumerate}
where $c[s/s']$ is the syntactic replacement in $c$ of subexpression~$s$ by~$s'$.
We assume that~$c[x/d] \in \C$ for any constraint atom~$c[x] \in \C$, variable~$x \in \X$ and~$d \in \Du$.

The flexibility of syntax and semantics of constraint atoms allows us to capture entities across different theories.
For instance, a propositional atom $\prAt$ (as understood in regular ASP) can also be represented as an \HTC-constraint atom
$\text{``}\prAt=\mathbf{t}\text{''}\in\C$ with a single variable $\prAt \in \X$ and the expected denotation:
\(
\den{\prAt=\mathbf{t}}=\{v\in\mathcal{V}\mid v(\prAt)=\mathbf{t}\}
\)
assuming we include a value $\mathbf{t} \in \D$ standing for ``true.''
We use letter $\prAt$ to denote variables for propositional constraint atoms and
we abbreviate ``$\prAt=\mathbf{t}$'' simply by ``$\prAt$'' in \HTC-formulas, when there is no ambiguity.
For another example,
the linear equation $x+y=4$
can be captured via constraint atoms of the same syntax ``$x+y= 4$'' whose denotation is
\begin{align*}
  \den{\text{``}x+y= 4\text{''}}
  =
  \{v\in\mathcal{V}\mid v(x),v(y), 4\in \mathbb{Z}, v(x)+v(y)= 4\}
  \ .
\end{align*}
Here, we have $\vars{\text{``}x-y=4\text{''}} = \{x,y\} \subseteq \X$.
Note that this constraint can only be satisfied if $x$ and $y$ hold an integer value and $4\in \mathbb{Z}$.

For clarity, we remove quotes from constraint atoms, when clear from the context.

A formula $\varphi$ over signature $\tuple{\X,\D,\C}$ is defined as
\begin{align*}
  \varphi::= \bot \mid c\mid \varphi \land \varphi \mid  \varphi \lor \varphi \mid  \varphi \rightarrow \varphi \quad\text{ where }c\in\C.
\end{align*}
We define $\top$ as $\bot \rightarrow \bot$ and $\neg\varphi$ as $\varphi \rightarrow \bot$ for any formula~$\varphi$.
We let $\vars{\varphi}$ stand for set of variables occurring in all constraint atoms in formula~$\varphi$.
A \emph{theory} is a set of formulas.

In \HTC,
an \emph{interpretation} over $\X,\D$ is a pair $\langle h,t \rangle$
of valuations over $\X,\D$ such that $h\subseteq t$.
The interpretation is \emph{total} if $h=t$.
Given a denotation $\den{\cdot}$,
  an interpretation $\langle h,t \rangle$ \emph{satisfies} a formula~$\varphi$,
  written $\langle h,t \rangle \models \varphi$,
  if \begin{enumerate}
  \item $\langle h,t \rangle \not\models \bot$
  \item $\langle h,t \rangle \models c \text{ if } h\in \den{c}$ \label{item:htc:atom}
  \item $\langle h,t\rangle \models \varphi \land \psi \text{ if }  \langle h,t\rangle \models \varphi \text{ and }  \langle h,t\rangle \models \psi$
  \item $\langle h,t\rangle \models \varphi \lor \psi \text{ if }  \langle h,t\rangle \models \varphi \text{ or }  \langle h,t\rangle \models \psi$
  \item $\langle h,t\rangle \models \varphi \rightarrow \psi
    \text{ if }\langle w,t\rangle \not\models \varphi \text{ or } \langle w,t\rangle \models \psi
    \text{ for }w\in\{h,t\}$
  \end{enumerate}
We say that an interpretation~$\tuple{h,t}$ is a model of a theory~$\Gamma$,
written $\tuple{h,t} \models \Gamma$,
when $\tuple{h,t} \models \varphi$ for every $\varphi \in \Gamma$.
We write $\Gamma \equiv \Gamma'$ if $\Gamma$ and $\Gamma'$ have the same models.
We omit braces whenever $\Gamma$ (resp.~$\Gamma'$) is a singleton.
A (total) interpretation $\langle t,t\rangle$ is an \emph{equilibrium model} of a theory~$\Gamma$,
if $\langle t,t\rangle \models \Gamma$ and there is no $h\subset t$
such that $\langle h,t\rangle \models \Gamma$.

One last comment regarding propositional constraint atoms ``$\prAt = \mathbf{t}$'' $\in \C$.
Although, in principle, we may have valuations assigning $\prAt$ any arbitrary value (like, for instance $v(\prAt)=4$),
in practice, there is no way to derive those values (assuming variable $\prAt$ is not used in other constraints),
so any equilibrium model $\tuple{t,t}$ eventually assigns either $t(\prAt)=\mathbf{t}$ or $t(\prAt)=\undefined$
(that, in this context, captures that $\prAt$ does not hold).

 \section{Logical Characterization of Answer Set Programming Modulo Theory}\label{sec:approach}

We now device a variant of \HTC\ for \tprogram{s} and
show that it corresponds to their original semantics in terms of a program transformation
(cf.~Section~\ref{sec:background};~\citeNP{gekakaosscwa16a}).
This provides us with an \HTC-based semantics for the theory reasoning framework of \clingo\ that allows for formal
elaborations of \tprogram{s}.
To this end, we introduce the concept of an \emph{abstract theory}~\AT\ which is our key instrument
for establishing more fine-grained formal foundations.
With it, we revisit the transformation-based approach and give a formal account of \AT-solutions in
Section~\ref{sec:translation:revisited}.

\subsection{Abstract theories}

An \emph{abstract theory}~\AT\ is a triple
\(
\langle \TheoryAtoms, \mathcal{S}, \comp{\cdot} \,\rangle
\)
where~\TheoryAtoms\ is a set of \emph{theory atoms},
$\mathcal{S} \subseteq 2^\TheoryAtoms$ is a set of \AT-satisfiable sets of theory atoms, and
$\comp{\cdot}:\TheoryAtoms \rightarrow\TheoryAtoms$ is a function mapping theory atoms to their \emph{complement}
such that $\comp{\comp{\thAt}}=\thAt$ for any $\thAt \in\TheoryAtoms$.
We define $\comp{S} = \{\comp{\thAt}\mid\thAt\in S\}$ for any set~$S\subseteq\TheoryAtoms$.
Note that the set $\mathcal{S}$ acts as an oracle whose rationality is beyond our reach.

Despite the limited structure of such theories,
some simple properties can be formulated:
A set $S\subseteq\TheoryAtoms$ of theory atoms is
\begin{itemize}
\item \emph{consistent}, if $\{\thAt,\comp{\thAt}\}\not\subseteq S$ for all $\thAt\in\TheoryAtoms$,
\item \emph{complete}, if $\thAt\in S$ or $\comp{\thAt}\in S$ for all $\thAt\in\TheoryAtoms$, and
\item \emph{closed}, if $\thAt\in S$ implies $\comp{\thAt}\in S$.
\end{itemize}
Accordingly, an abstract theory $\AT = \langle \TheoryAtoms, \mathcal{S}, \comp{\cdot} \,\rangle$
is \emph{consistent} or \emph{complete} if all its \mbox{\AT-satisfiable} sets $S\in\mathcal{S}$ are \emph{consistent} or
\emph{complete}, respectively.
Note that any closed set $S\neq \emptyset$ is inconsistent but can be $\AT$-\nolinebreak{}satisfiable (e.g., for paraconsistent theories).
We mostly use closed sets to describe subtypes of theory atoms, rather than satisfiable solutions.

As an example, consider an abstract theory $\LC=\langle\TheoryAtoms,\mathcal{S},\comp{\cdot}\rangle$ of linear equations,
where
\begin{itemize}
\item \TheoryAtoms\ is the set of all expressions of form \eqref{clingcon:linear:constraint},
\item $\mathcal{S}$ is a set of subsets $S\subseteq\TheoryAtoms$ of expressions of form \eqref{clingcon:linear:constraint}, and
\item the complement function is defined as
$\comp{\code{\&sum\{\dots\}} \prec c}=\code{\&sum\{\dots\}}\mathrel{\comp{\prec}}c$
  with
  $\comp{\code{<=}}\eqdef \code{>}$,
  $\comp{\code{=}}\eqdef\code{!=}$,
  $\comp{\code{!=}}\eqdef\code{=}$,
  $\comp{\code{<}}\eqdef\code{>=}$,
  $\comp{\code{>}}\eqdef\code{<=}$, and
  $\comp{\code{>=}}\eqdef\code{<}$.
\end{itemize}
Note that the set of theory atoms in \LC\ is closed.
Although we expect theories of linear equations to be consistent,
we have yet no means to establish such a property.

However,
we may make sense of certain theories, like linear equations, by associating a rational way of constructing
satisfiable sets of theory atoms.
To this end,
we relate abstract theories to a denotational semantics similar to~\HTC\ (although other choices exist).
Such theory-specific structures allow us to establish properties of abstract theories
and ultimately characterize their integration into ASP.
Given an abstract theory~$\AT=\langle \TheoryAtoms, \mathcal{S}, \comp{\cdot} \,\rangle$,
we define a \emph{structure} as a tuple
\(
(\X_\AT,\D_\AT, \varsATfunction, \den{\cdot}_\AT),
\)
where
\begin{enumerate}
\item $\X_\AT$ is a set of variables,
\item $\D_\AT$ is a set of domain elements,
\item ${\varsATfunction: \TheoryAtoms \rightarrow 2^{\X_\AT}}$ is a function giving the set of variables contained in a theory atom
  such that~$\varsATfunction(\thAt) = \varsATfunction(\comp{\thAt})$ for all theory atoms~$\thAt\in\TheoryAtoms$,
\item $\V_\AT=\{v \mid v: \X_\AT \rightarrow \D_\AT\}$ is the set of all valuations over~$\X_\AT$ and~$\D_\AT$, and
\item $\den{\cdot}_\AT: \TheoryAtoms \rightarrow 2^{\V_\AT}$ is a function mapping
  theory atoms to sets of valuations such that
  $$v \in \den{\thAt}_\AT \text{ iff } w \in \den{\thAt}_\AT$$
  for all theory atoms~$\thAt\in\TheoryAtoms$ and every pair of valuations~$v,w$ agreeing on the value of all variables~$\varsAT{\thAt}$ occurring in~$\thAt$.
\end{enumerate}
Whenever an abstract theory \AT\ is associated with such a structure, we call it \emph{structured}
(rather than abstract).

Given a set $S$ of theory atoms, we define its denotation as $\den{S}_\AT \eqdef \bigcap_{\thAt\in S}\den{\thAt}_\AT$.

We define a theory~$\AT=\langle \TheoryAtoms, \mathcal{S}, \comp{\cdot} \,\rangle$
structured by $(\X_\AT,\D_\AT,\varsATfunction, \den{\cdot}_\AT)$
as \emph{compositional},
if
$\mathcal{S} = \{S \subseteq \TheoryAtoms \mid \den{S}_\AT \neq \emptyset\}$, that is, a set $S$ is $\AT$-satisfiable iff its denotation is not empty.

As an example,
let us associate the theory of linear equations \LC\ with the structure $(\X_\LC,\D_\LC, \varsTfunction{\LC},\den{\cdot}_\LC)$, where
$\X_\LC$ is an infinite set of integer variables,
$\D_\LC=\mathbb{Z}$,
$\varsT{\LC}{\code{\&sum\{\mathit{k_1*x_1};\dots;\mathit{k_n*x_n}\}} \prec k_0}=\{x_1,\cdots,x_n\}$, and
\begin{multline}
  \den{\code{\&sum\{\mathit{k_1*x_1};\dots;\mathit{k_n*x_n}\}} \prec k_0}_\LC
  =\\\textstyle
  \{v\in\mathcal{V}_\LC \mid \{k_1,v(x_1),\dots k_m,v(x_m)\}\subseteq\mathbb{Z}, \sum_{1\leq i \leq n}k_i*v(x_i)\prec k_0\}.
\end{multline}
With \LC,
a set $S$ of theory atoms capturing linear equations is \LC-satisfiable, if
\(
\den{S}_\LC
\)
is non-empty.
Once \LC\ is structured this way, we can establish the following properties.
\begin{proposition}\label{prop:lc:compositional}\label{prop:lc:consistent}
  The theory \LC\ structured by $(\X_\LC,\D_\LC,\varsTfunction{\LC}, \den{\cdot}_\LC)$ is compositional and consistent.
\end{proposition}

Whenever a theory \AT\ is compositional,
we can define an associated entailment relation $\models_\AT$ so that $S \models_\AT \thAt$ when $\den{S}_\AT \subseteq \den{\thAt}_\AT$.
If $S$ is a singleton, we omit brackets.
For instance, it is easy to see that $\code{\&sum\{x;y\}}\geq 1 \models_\LC \code{\&sum\{x;y\}}\geq 0$ since any integer valuation $v$ such that $v(x)+v(y)\geq 1$ must satisfy $v(x)+v(y)\geq 0$ as well.
It is easy to see that this entailment relation is monotonic, that is, $S \models_\AT \thAt$ implies $S \cup S' \models_\AT \thAt$.
This is because, when $\den{S}_\AT \subseteq \den{\thAt}_\AT$, we obtain $\den{S \cup S'}_\AT = \den{S}_\AT \cap \den{S'}_\AT \subseteq \den{\thAt}_\AT$ too.
Notice that, in general, most non-monotonic formalisms are non-compositional in the sense that their satisfiability condition $S \in \mathcal{S}$ usually depends on the whole set $S$ of theory atoms and cannot be described in terms of the individual satisfiability of each atom $\thAt \in S$.
Compositional theories are interesting from an implementation point of view, since they allow for handling partial assignments that can be extended monotonically.

Some compositional theories have a complement $\comp{\thAt}$ whose denotation is precisely the set-complement of $\den{\thAt}_\AT$, that is, $\den{\comp{\thAt}}_\AT = \V_\AT \setminus \den{\thAt}_\AT$.
In this case, we say that the complement is \emph{absolute}.
The compositional theory of linear constraints \LC\ has an absolute complement:
For instance, $\den{\comp{\code{\&sum\{x;y\}=4}}}=\den{\code{\&sum\{x;y\}!=4}}=\V_\LC \setminus \den{\code{\&sum\{x;y\}=4}}$.
Examples of non-absolute complements may arise when we deal, for instance, with multi-valued logics (like Kleene's or \L ukasiewicz')
where we may have valuations that are not models of a formula nor its complement (assuming we use negation for that role).
Having an absolute complement directly implies that the theory is consistent.
This is because the denotations $\den{\thAt}_\AT$ and $\den{\comp{\thAt}}_\AT$ are disjoint, and so, any set including $\{\thAt,\comp{\thAt}\}$ is \AT-unsatisfiable.
\begin{proposition}\label{prop:entailment}
  For any compositional theory $\AT=\langle \TheoryAtoms, \mathcal{S}, \comp{\cdot} \,\rangle$
  with an absolute complement $\comp{\cdot}$,
  any $S \subseteq \TheoryAtoms$ and $\thAt \in \TheoryAtoms$,
  we have that
  $S \models_\AT\thAt$ iff $(S \cup \{\comp{\thAt}\}) \not\in \mathcal{S}$.
\end{proposition}
 \subsection{Transformation-based semantics revisited}
\label{sec:translation:revisited}

With a firm definition of what constitutes a theory \AT,
we can now make precise the definition of a \AT-solution and \AT-stable model of a \tprogram\ given in~\eqref{eq:transformation}.
For clarity, we refine those definitions by further specifying the subset of external theory atoms in \TheoryAtomsES:
Given an abstract theory $\AT = \langle \TheoryAtoms, \mathcal{S}, \comp{\cdot} \,\rangle$ and
a set~${\TheoryAtomsES \subseteq \TheoryAtoms}$ of external theory atoms,
we define $S\in\mathcal{S}$ as a \emph{$\langle\AT,\TheoryAtomsES\rangle$-solution},
if $S\cup(\comp{\TheoryAtomsES\setminus S})$ is \AT-satisfiable.
This set $\comp{\TheoryAtomsES\setminus S}$ acts as a ``completion'' of $S$, adding all complement atoms $\comp{\thAt}$ for every external atom $\thAt$ that does not occur explicitly in $S$.
Sometimes, this completion does not add any atom to $S$: this happens when, for each external $\thAt \in \TheoryAtomsES$, either $\thAt \in S$, $\comp{\thAt} \in S$ or both.
We say that a set of theory atoms $S$ is \TheoryAtomsES-\emph{complete} if $S\cup(\comp{\TheoryAtomsES\setminus S})=S$ and \TheoryAtomsES-\emph{incomplete} otherwise.
For instance, consider an abstract theory from \LC\ with the theory atoms
$\TheoryAtoms=\{\thAt_1,\thAt_2,\thAt_3,\thAt_4\}$ with $\thAt_1=(\sumxy)$, $\thAt_2=(\sumyz)$ given above,
plus $\thAt_3=(\sumxyc)$, $\thAt_4=(\sumyzc)$, so that $(\thAt_1,\thAt_3)$ and $(\thAt_2,\thAt_4)$ are pairwisely complementary.
Suppose we only have one external atom $\TheoryAtomsES=\{\thAt_1\}$.
Then set $S=\{\thAt_2\}$ is \TheoryAtomsES-incomplete since $S \cup(\comp{\TheoryAtomsES\setminus S})=\{\thAt_2,\comp{\thAt_1}\}=\{\thAt_2,\thAt_3\}$ is completed with $\thAt_3$.
On the other hand, set $S=\{\thAt_1,\thAt_2\}$ is \TheoryAtomsES-complete because we do have information about the external atom $\thAt_1$, and so, $\comp{\TheoryAtomsES\setminus S}$ does not provide any additional information.
The next result shows that \TheoryAtomsES-complete solutions suffice when considering the existence of solution:
\begin{proposition}\label{prop:complete}
  Let $\AT = \langle \TheoryAtoms, \mathcal{S}, \comp{\cdot} \,\rangle$ be an abstract theory
  and $\TheoryAtomsES\subseteq \TheoryAtoms$ be a a set of external atoms.

  For any \emph{\tsolution} $S\subseteq\TheoryAtoms$, we have:
  \begin{enumerate}
  \item $S'=S\cup(\comp{\TheoryAtomsES\setminus S})$ is
    \TheoryAtomsES-complete and $S'$ is also a \emph{\tsolution}.\label{prop:complete:tsoltocomp}
  \item If $S$ is \TheoryAtomsES-complete, then $S$ is $\AT$-satisfiable.\label{prop:complete:comptosat}
  \end{enumerate}
\end{proposition}
For consistent theories, completing a set $S$ with $\comp{\TheoryAtomsES\setminus S}$ may be convenient for any external atom~$\thAt$ whose complement $\comp{\thAt}$ is not external, as happened with $\thAt_1=(\sumxy) \in \TheoryAtomsES$ and $\comp{\thAt_1}=(\sumxyc) \not\in \TheoryAtomsES$ in our previous example.
However, if the complement, in its turn, is also external and not in $S$, then we also complete $S$ with $\comp{\comp{\thAt}}=\thAt$,
and so, the completed set is inconsistent.
If, additionally, $\TheoryAtomsES$ is closed, it contains all the complements of its elements, and so:
  \begin{proposition}
    \label{prop:consistentclosed}
    Let $\AT = \langle \TheoryAtoms, \mathcal{S}, \comp{\cdot} \,\rangle$ be a consistent abstract theory with a \emph{closed} set of external atoms $\TheoryAtomsES\subseteq \TheoryAtoms$.
Then, all \tsolution s $S\subseteq\TheoryAtoms$ are \TheoryAtomsES-complete.
  \end{proposition}

Back to our example, if we take now $\TheoryAtomsES=\{\thAt_1,\thAt_3\}$ as external atoms, note that this set is closed since $\comp{\thAt_1}=\thAt_3$ and $\comp{\thAt_3}=\thAt_1$.
The only possibility for an \TheoryAtomsES-incomplete solution~$S$ would be that none of these two atoms were included in~$S$.
But then, we would have to complete with $\comp{\TheoryAtomsES\setminus S}=\{\thAt_1,\thAt_3\}$ and the result would be inconsistent,
since these two atoms complement one another.

In analogy to Section~\ref{sec:background},
a set $X\subseteq \Atoms\cup\TheoryAtoms$ of atoms is a $\langle\AT,\TheoryAtomsES\rangle$-\emph{stable model} of a \tprogram\ $P$,
if there is some \tsolution\ $S$ such that $X$ is a stable model of the program in~\eqref{eq:transformation}.
As an illustration, take again \conprogram\ $P_{(\ref{ex:clingcon:rule:one}/\ref{ex:clingcon:rule:two})}$
with abstract theory \LC,
the already seen theory atoms $\TheoryAtoms=\{\thAt_1,\thAt_2,\thAt_3,\thAt_4\}$, and
the closed subset of external atoms $\TheoryAtomsES=\{\thAt_1,\thAt_3\}$.
This program has two $\tuple{\LC,\TheoryAtomsES}$-stable models:
\begin{align*}
  X_1 &= \{\code{a},\thAt_1,\thAt_2\} = \{\code{a},\sumxy,\sumyz\}\text{ and}\\
  X_2 &= \{\thAt_3\} = \{\sumxyc\}
\end{align*}
To verify $X_1$, take the \lsolution\ $\{\thAt_1,\thAt_2\}=\{\code{\&sum\{x;y\}=4},\code{\&sum\{y;z\}=2}\}$ and the resulting program transformation
\begin{align}
  \code{a}              &\leftarrow \code{\&sum\{x;y\}=4}\nonumber\\
  \code{\&sum\{y;z\}=2} &\leftarrow \code{a}             \label{ex:con2}\\
  \code{\&sum\{x;y\}=4} &\leftarrow                      \nonumber
\end{align}

We see that $X_1$ is a stable model of the logic program and, as such, also a $\tuple{\LC,\TheoryAtomsES}$-stable model.
Similarly, with \lsolution\ $\{\thAt_3\}=\{\sumxyc\}$, we get program
\begin{align}
  \code{a}              &\leftarrow \code{\&sum\{x;y\}=4} \nonumber\\
  \code{\&sum\{y;z\}=2} &\leftarrow \code{a}              \nonumber\\
  \code{\&sum\{x;y\}!=4}&\leftarrow                       \label{ex:con3}\\
                        &\leftarrow \code{\&sum\{y;z\}=2} \nonumber\\
                        &\leftarrow \code{\&sum\{y;z\}!=2}\nonumber
\end{align}
Again, we have $X_2$ as a stable model, confirming it as a $\tuple{\LC,\TheoryAtomsES}$-stable model.
Notice that, in this case, the \tsolution\ $\{\thAt_1\}$ is not unique:
We could have also freely added any of the two founded atoms $\thAt_3=(\sumyz)$ or $\thAt_4=(\sumyzc)$ so that one of their respective constraints (the last two lines in the program above) would not be included in each case.
This shows that, for founded theory atoms not derived by any rule,
the represented linear equation, its complement, or none of the two are free to hold.
 \subsection{\AMT\ semantics based on \HTC}
\label{sec:htcsemantics}

We now present a direct encoding of a \tprogram\ $P$ as an \HTC\ theory.
This encoding is ``direct'' in the sense that it preserves the structure of $P$ rule by rule and atom by atom, only requiring the addition of a fixed set of axioms, one per each external atom.
As a first step, we start embodying compositional structured theories in \HTC.
To this end, we observe that the definitions of structured theories \AT\ and \HTC\ deal with quite similar concepts
(viz., a domain, a set of variables, valuation functions, and denotations for atoms).
So, in principle, it seems reasonable to establish a one-to-one correspondence.
However, the generality of \HTC\ allows us to go further by
tolerating different abstract theories in the same formalization.
For this reason, when we encode a compositional theory
$\AT= \langle \TheoryAtoms, \mathcal{S},\comp{\cdot}\ \rangle$ with a structure $(\X_\AT,\D_\AT,\varsATfunction,\den{\cdot}_\AT)$ into an
\HTC\ theory over a signature $\tuple{\X,\D,\C}$,
we just require $\X_\AT \subseteq \X$ and $\D_\AT \cup \{ \mathbf{t} \} \subseteq \D$,
so that the \HTC\ signature may also include variables and domain values other than the ones in \AT.
In particular, we assign the value~$\mathbf{t}$ to propositional variables that are true.
We then map each abstract theory atom $\thAt \in \TheoryAtoms$ into a corresponding \HTC-constraint atom $\htcsemanticsA(\thAt) \in \C$
with~$\vars{\htcsemanticsA(\thAt)} = \varsT{\AT}{\thAt}$.
We also require that the \HTC\ denotation satisfies:
\begin{align}
\den{\htcsemanticsA(\thAt)} \ \eqdef \
    \ \{v\in\V_{\X,\D} \mid \exists w\in\den{\thAt}_\AT, \ \restr{v}{\vars{\htcsemanticsA(\thAt})}=\restr{w}{\varsT{\AT}{\thAt}} \} \label{f:atomden}
\end{align}
Note that \HTC\ valuations $v\in\V_{\X,\D}$ are applied to a (possibly) larger set of variables $\X \supseteq \X_\AT$ but also range on a larger set of domain values $\Du \supset \D_\AT$, including the element $\undefined \not\in \D_\AT$ to represent \HTC-undefined variables.
The denotation $\den{\htcsemanticsA(\thAt)}$ collects all possible \HTC-valuations that coincide with some \AT-valuation $w \in \den{\thAt}_\AT$ on the atom's variables $\varsT{\AT}{\thAt}$, letting everything else vary freely.
We write $\htcsemanticsA(S)$ for $\{ \htcsemanticsA(\thAt) \mid \thAt \in S \}$ for $S \subseteq \TheoryAtoms$.

A first interesting result shows that this mapping of denotations preserves \AT-satisfiability:
\begin{proposition}\label{prop:translation.satisfiable}
  Given a compositional theory~$\AT= \langle \TheoryAtoms, \mathcal{S}, \comp{\cdot}\ \rangle$,
  a set~$S \subseteq \TheoryAtoms$ of theory atoms is \AT-satisfiable iff~$\htcsemanticsA(S)$ is satisfiable in~$\HTC$.
\end{proposition}

Let us now consider a compositional theory \AT\ in a \tprogram\ $P$ over~$\langle\Atoms,\TheoryAtoms,\TheoryAtomsES\rangle$
where \Atoms\ is a set of propositional atoms and $\TheoryAtomsES \subseteq \TheoryAtoms$ a subset of external theory atoms.
As with atoms in $\TheoryAtoms$,
we also encode each propositional atom $\regAt \in \Atoms$ as an \HTC-constraint atom
$\htcsemanticsA(\regAt) \eqdef \text{``}\prop{\regAt}=\mathbf{t}\text{''}$ in \C,
assuming we have a variable $\prop{\regAt} \in \X$ for each $\regAt \in \Atoms$.

We write $\htcsemanticsA(P)$ to stand for the atom-level translation of a \tprogram\ $P$.
That is, $\htcsemanticsA(P)$ is a theory containing one implication
\begin{align}\label{eq:htcsemantics:rule}
  &\htcsemanticsA(\anyAt_1)\land\dots\land \htcsemanticsA(\anyAt_n)\land\neg \htcsemanticsA(\anyAt_{n+1})\land\dots\land\neg \htcsemanticsA(\anyAt_{m})\rightarrow \htcsemanticsA(\anyAt_0)
\end{align}
for each rule $r\in P$ of form~\eqref{theory:rule}.
Additionally, when $\anyAt_0=\bot$ in \eqref{eq:htcsemantics:rule},
the atom translation is simply $\htcsemanticsA(\bot) \eqdef \bot$.
The complete translation of \tprogram~$P$ over~$\langle\Atoms,\TheoryAtoms,\TheoryAtomsES\rangle$
for compositional theory \AT\ (as defined above) into \HTC\
is denoted as~$\htcsemanticsP(P,\AT,\TheoryAtomsES)$ and
given by the union of $\htcsemanticsA(P)$ plus a formula
\begin{align}\label{eq:htcsemantics:external}
  \htcsemanticsA(\thAt) \lor \htcsemanticsA(\comp{\thAt})
    \qquad
    \text{ for each external theory atom }\thAt \in \TheoryAtomsES.
\end{align}
\begin{theorem}[Main Result]\label{thm:first_translation}
  Given a \tprogram\ $P$ over $\langle\Atoms,\TheoryAtoms,\TheoryAtomsES\rangle$ with \TheoryAtomsES\  closed, and
  a consistent, compositional theory $\AT=\langle \TheoryAtoms, \mathcal{S}, \comp{\cdot} \,\rangle$,
  there is a one-to-many correspondence between the $\langle\AT,\TheoryAtomsES\rangle$-stable models of~$P$ and
  the equilibrium models of~$\htcsemanticsP(P,\AT,\TheoryAtomsES)$ in \HTC\ such that
  $X$ is a $\langle\AT,\TheoryAtomsES\rangle$-stable model of $P$ iff
  there exists an equilibrium model~$\tuple{t,t}$ of theory~$\htcsemanticsP(P,\AT,\TheoryAtomsES)$
  that satisfies
  \begin{gather}\label{eq:1:thm:first_translation}
    X \ = \ \big\{ \ \anyAt\in \Atoms\cup\TheoryAtomsES \ \mid \ t \in \den{\htcsemanticsP(\anyAt)} \ \big\}
    \, \cup \, \big\{\, \thAt \in \TheoryAtomsDN \, \mid \, (B \to \thAt) \in P    \text{ and } \tuple{t,t} \models \htcsemanticsP(B) \ \big\}
  \end{gather}
  where~$\htcsemanticsP(B)$ stands for the result of applying~$\htcsemanticsP$ to all atoms occurring in conjunction~$B$.
\end{theorem}
The semantics of a \tprogram{} $P$ is then given by the equilibrium models of $\htcsemanticsP(P,\AT,\TheoryAtomsES)$ in \HTC.
Theorem~\ref{thm:first_translation} states that this semantics remains faithful to the program transformation.
Intuitively, formulas like $\eqref{eq:htcsemantics:rule}$ capture the rules in the \tprogram\ and are used for the same purpose, that is, to decide which founded atoms from $\TheoryAtoms \setminus \TheoryAtomsES$ can be eventually derived.
Furthermore, due to the minimization imposed to obtain an equilibrium model $\tuple{t,t}$,
if a founded atom $\thAt$ is not derived (that is, $\tuple{t,t} \not\models \htcsemanticsA(\thAt)$),
then all its variables $x \in \varsT{\AT}{\thAt}$ not occurring in external atoms are left undefined, $t(x)=\undefined$.
On the other hand,
the axiom \eqref{eq:htcsemantics:external} for external atoms $\thAt\in\TheoryAtomsES$
acts as a stronger version of the usual choice construct $\htcsemanticsA(\thAt)\vee\neg\htcsemanticsA(\thAt)$ in \HT.
That is, we can freely add $\htcsemanticsA(\thAt)$ or not but, when the latter happens,
we further provide evidence for the complement~$\htcsemanticsA(\comp{\thAt})$.

As an example, take theory \LC\ with structure $(\X_\LC,\D_\LC,\varsTfunction{\LC}, \den{\cdot}_\LC)$, and
assume that we represent each linear equation $\code{e}=\code{\&sum\{\mathit{k_1*x_1};\dots;\mathit{k_n*x_n}\}} \prec k_0$
as the \HTC\ constraint atom $\htcsemanticsA(\code{e}) =$ ``$k_1 \cdot x_1 + \dots + k_n \cdot x_n \prec k_0$'' with the following denotation:
\begin{align}
  \den{\htcsemanticsA(\code{e})}&=\{v\in\mathcal{V}_{\X,\D} \mid v'\in\den{\code{e}}_\LC, \ \restr{v}{\vars{\htcsemanticsA(\code{e})}}=\restr{v'}{\varsT{\LC}{\code{e}}}\}       \nonumber\\
                                &=\textstyle\{v\in\mathcal{V} \mid \{k_1,v(x_1),\dots k_m,v(x_m)\}\subseteq\mathbb{Z}, \ \sum_{1\leq i \leq n}k_i \cdot v(x_i)\prec k_0\}\label{f:denotLC}
\end{align}
where $\prec$ is associated with its standard mathematical relation.
Essentially, the denotation selects the variables that are relevant to theory \LC\ and the linear equation atom $\code{e}$, and
then applies the denotation $\den{\code{e}}_\LC$ given by the structure.
Then, in our \conprogram\ $P_{(\ref{ex:clingcon:rule:one}/\ref{ex:clingcon:rule:two})}$,
defining the external atoms \TheoryAtomsES\ as $\{(\sumxy), (\sumxyc)\}$,
produces the \HTC\ theory
$\htcsemanticsP(P_{(\ref{ex:clingcon:rule:one}/\ref{ex:clingcon:rule:two})},\LC,\TheoryAtomsES)$
containing
\begin{align}
x+y=4&\rightarrow \prop{\code{a}}, \label{f:tauP1}\\
\prop{\code{a}} &\rightarrow y+z=4, \label{f:tauP2}\\
(x+y=4)&\lor (x+y \neq 4) \label{f:choice42}
\end{align}
where \eqref{f:choice42} corresponds to the choice \eqref{eq:htcsemantics:external} for the external atoms.
Interestingly, in the setting of \conprogram s, we can replace choices as in \eqref{eq:htcsemantics:external} on external atoms by disjunctions of constraints for their variables, forcing them to take any arbitrary integer value.
In our example, we can replace \eqref{f:choice42} by the disjunctions of \LC-atoms:
\begin{align}
x\geq 0 \vee x<0 \label{f:def1}\\
y\geq 0 \vee y<0 \label{f:def2}
\end{align}
These formulas are no tautologies:
They assign any arbitrary pair of \emph{integer values} to the variables $x$ and $y$.
To conclude with the example,
the theory $\htcsemanticsP(P_{(\ref{ex:clingcon:rule:one}/\ref{ex:clingcon:rule:two})},\LC,\TheoryAtomsES)$
has an infinite number of equilibrium models $\tuple{t,t}$ that satisfy one of the two following conditions:
\begin{enumerate}
\item $t(\prop{\code{a}})=\mathbf{t}$, \ $\{t(x),t(y),t(z)\}\subseteq\mathbb{Z}$, \ $t(x)+t(y)=4$ \ and \ $t(y)+t(z)=4$, or
\item $t(\prop{\code{a}})=\undefined$, \ $t(z)=\undefined$, \ $\{t(x),t(y)\}\subseteq\mathbb{Z}$ \ and \ $t(x)+t(y) \neq 4$.
\end{enumerate}

 \section{Answer Set Solving modulo Linear Equations}
\label{sec:hyrbid:systems}

Finally, we show how our formalism can be used to capture the semantics of several \clingo\ extensions with linear equations.
At first, we use the structured theory \LC\ to describe a semantics of \clingcon\ using \HTC-theories.
Then, we introduce structured theories \DLAT\ and \LP\ to analogously capture \clingoM{dl} and \clingoM{lp}.
Among others, this allows us to compare the systems and show strongly equivalent program transformations.
We start by introducing a useful constraint atom $\dfZ{x}\in\C$ to represent the fact that a given variable
$x\in\X$ has a defined, integer value, viz.\ $\den{\dfZ{x}}=\{v\in\V \mid v(x)\in \mathbb{Z}\}$.
It is not difficult to see that $\den{\dfZ{x}}$ is \HTC\ equivalent to \eqref{f:def1}.
\begin{proposition}\label{prop:linear:defined}
  Let $P$ be a \tprogram\ over $\tuple{\Atoms,\TheoryAtoms,\TheoryAtomsES}$ wrt
  theory $\LC=\tuple{\TheoryAtoms,\mathcal{S},\comp{\cdot}\,}$ structured by $(\X_\LC,\D_\LC,\varsTfunction{\LC}, \den{\cdot}_\LC)$,
  then the following three theories are strongly equivalent:
  \begin{itemize}
    \item $\htcsemanticsP(P,\LC,\TheoryAtoms)$,
    \item $\htcsemanticsP(P) \cup \{\htcsemanticsA(\thAt) \lor \neg\htcsemanticsA(\thAt) \mid \thAt \in \TheoryAtomsES\} \cup \{ \dfZ{x} \mid x\in\varsT{\LC}{\thAt}, \thAt \in \TheoryAtomsES \}$, and
    \item $\htcsemanticsP(P) \cup \{ \dfZ{x} \mid x\in\varsT{\LC}{\thAt}, \thAt \in \TheoryAtomsES \}$
  \end{itemize}
\end{proposition}
In essence, this means that the choice \eqref{eq:htcsemantics:external} in $\htcsemanticsP(P,\LC,\TheoryAtoms)$ can be safely replaced by a set of constraint atoms $\dfZ{x}$ for every variable $x$ that occurs in at least one external atom.

As an interesting result, we may observe that constraint atoms involving defined variables can always be rephrased as formulas in the scope of negation:
\begin{proposition}\label{prop:exvars}
  Let $\Gamma$ be a \HTC\ theory over signature $\tuple{\X,\D,\C}$ with a denotation for linear constraint atoms denoted as defined~\eqref{f:denotLC} and
  let
  \(
  c = \text{``}k_1 \cdot x_1 + \dots + k_n \cdot x_n \prec k_0\text{''} \in \C
  \)
  and
  \(
  \comp{c} = \text{``}k_1 \cdot x_1 + \dots + k_n \cdot x_n \comp{\prec} k_0\text{''} \in \C
  \)
  be two linear constraint atoms such that $\Gamma \models \dfZ{x_i}$ for $i = 1 \dots n$.
Then, the following strongly equivalent transformations hold:
  \begin{itemize}
    \item $\Gamma \models c \leftrightarrow \neg \neg c$,
    \item $\Gamma \models c \leftrightarrow \neg \comp{c}$,
    \item $\Gamma \models (F \to c) \leftrightarrow  (F \to \neg \neg c)$,
    \item $\Gamma \models (F \to c) \leftrightarrow (F \wedge \neg c \to \bot)$, and
    \item $\Gamma \models (F \to c) \leftrightarrow (F \wedge \comp{c} \to \bot)$
  \end{itemize}
\end{proposition}

In \clingcon, all atoms are external $\TheoryAtomsES = \TheoryAtoms$ and may indistinctly occur in the head or in the body.
To illustrate the behavior of this system, let us analyze the effect on our running example $P_{(\ref{ex:clingcon:rule:one}/\ref{ex:clingcon:rule:two})}$.
Since atom $y+z=2$ is also external now, the translated \HTC\ theory $\htcsemanticsP(P,\LC,\TheoryAtoms)$ would simply add the axiom:
\begin{eqnarray}
(y+z=2) \vee (y+z\neq 2) \label{f:choice42b}
\end{eqnarray}
to the previously obtained formalization \eqref{f:tauP1}-\eqref{f:choice42}.
According to Proposition~\ref{prop:linear:defined}, we can even rephrase this theory as $\{\eqref{f:tauP1},\eqref{f:tauP2}\} \cup \{\dfZ{x},\dfZ{y},\dfZ{z}\}$, that is, replacing the choices by constraint atoms that force all variables to be assigned some integer value.
As a result, the equilibrium models~$\tuple{t,t}$ always satisfy $\{t(x),t(y),t(z)\}\subseteq\mathbb{Z}$ plus one of these two conditions:
\begin{itemize}
\item $t(x)+t(y)=4$, $t(z)+t(y)=2$, and $t(\prop{\code{a}})=\mathbf{t}$
\item $t(x)+t(y) \neq 4$, and $t(\prop{\code{a}})=\undefined$
\end{itemize}
Both atoms $x+y=4$ and $y+z=2$ are external.
When $x+y=4$ holds, rule \eqref{f:tauP1} forces $\prop{\code{a}}$ to be true and rule \eqref{f:tauP2} further implies $y+z=2$.
For the second item, when $x+y=4$ does not hold, then its complement $x+y \neq 4$ becomes true and $y+z=2$ is also free to hold or not as in~\eqref{f:choice42b}.
However, since $y$ and $z$ are always assigned some integer value, this choice becomes tautological and does not impose any additional restriction on those variables.
On the other hand, there is no reason to derive $\prop{\code{a}}$ and so it is left undefined (the corresponding program atom $\code{a}$ does not hold).

One interesting feature of \clingcon\ is that it not only shows the $\tuple{\LC,\TheoryAtoms}$-stable models but also allows enumerating variable assignments for that stable model.
For instance, for $P_{(\ref{ex:clingcon:rule:one}/\ref{ex:clingcon:rule:two})}$, and $\code{a} \in X$ in the stable model, we may get $\{x=2, y=2, z=0\}$, $\{x=3, y=1, z=1\}$, $\{x=4, y=0, z=2\}$, etc.
When considering these witnesses, we actually obtain a one-to-one correspondence to the equilibrium models of the \HTC\ translation $\htcsemanticsP(P,\LC,\TheoryAtoms)$.

Since, in \clingcon, all variables are defined ($\dfZ{x}$ for all $x \in \X_\LC$), we can \emph{always} apply Proposition~\ref{prop:exvars} to constraint atoms in the head and shift them to the body.
As an example, we can safely replace rule \eqref{ex:clingcon:rule:two} in program $P_{(\ref{ex:clingcon:rule:one}/\ref{ex:clingcon:rule:two})}$ by any of the two constraints below, that are equivalent when all variables are defined:
\begin{align}
  \label{ex:clingcon:rule:twob}
  \bot &\ \leftarrow \ \neg \code{\&sum\{ y; z \} = 2}, \code{a}\\
  \label{ex:clingcon:rule:twoc}
  \bot &\ \leftarrow \ \code{\&sum\{ y; z \} != 2}, \code{a}
\end{align}

To cover \clingoM{dl}-programs,
we introduce abstract theory \DLAT\ capturing \emph{difference constraints} over integers,
which is a subset of the already seen abstract theory \LC\ where theory atoms have the fixed form
$\code{\&sum\{1 * \mathit{x}; (-1)*\mathit{y}\} <= \mathit{k}}$
but rewritten instead as:
\begin{align}\label{clingodl:difference:constraint}
\diffc{x}{y}{k}
\end{align}
where $x$ and $y$ are integer variables and $k \in\mathbb{Z}$.

As we have in \LC,
the complement $\comp{\eqref{clingodl:difference:constraint}}$ is absolute and, in this case, it is not difficult to see that its denotation corresponds to:
\[
  \den{\comp{\eqref{clingodl:difference:constraint}}}_{\DLAT} \ = \ \den{\ \diffc{y}{x}{-k-1} \ }_{\DLAT}\
\]
In \clingoM{dl}, body atoms are external $\TheoryAtomsES$ and head atoms founded $\TheoryAtoms\setminus \TheoryAtomsES$.
In this case, we have a one-to-one correspondence between the answers of the system and the $\tuple{\DLAT,\TheoryAtomsES}$-stable models of a \clingoM{dl}-logic program $P$.

Although, in \clingoM{dl}, we may have undefined variables, Proposition~\ref{prop:exvars} is still applicable if the variables in the head atom are used in other external atoms.
To put an example, take the \clingoM{dl}-program
\begin{eqnarray}
\code{margin} & \leftarrow & \diffc{x}{y}{10} \label{f:exdl1}\\
\diffc{x}{y}{0} & \leftarrow & \neg \code{margin} \label{f:exdl2} \\
 \diffc{y}{x}{0}  & \leftarrow & \neg \code{margin} \label{f:exdl3}
\end{eqnarray}
that says that we have some \code{margin} when picking values with $x-y\leq 10$, but we force $x=y$ if there is no such \code{margin}.
Under the assumption of external body atoms, this program is strongly equivalent to:
\begin{eqnarray*}
\code{margin} & \leftarrow & \diffc{x}{y}{10}\\
\bot & \leftarrow & \diffc{y}{x}{-1}, \ \neg \code{margin}\\
\bot & \leftarrow & \diffc{x}{y}{-1}, \ \neg \code{margin}
\end{eqnarray*}
since variables $x$ and $y$ occur in an external (body) atom $\code{\&diff\{\mathit{x}-\mathit{y}\} \ \mbox{\tt <=} \ \mathit{10}}$, and so,
they are always defined.
Our \HTC\ formalization can be exploited now to prove other strong equivalence relations whose proof was non-trivial before.
For instance, we can prove that adding the rule
\begin{eqnarray*}
\diffc{z}{y}{20} & \leftarrow & \neg \code{margin}
\end{eqnarray*}
to the \clingoM{dl}-program \eqref{f:exdl1}-\eqref{f:exdl3} is strongly equivalent to adding instead the rule
\begin{eqnarray*}
\diffc{z}{x}{20} & \leftarrow & \neg \code{margin}
\end{eqnarray*}
where we use $x$ rather than $y$ in the head, since these two variables are always defined and $\neg \code{margin}$ forces them to be equal.

A third \AMT\ system covered by our formalization is \clingoM{lp}.
In this case, the abstract theory \LP\ is about linear equations over reals.
It is identical to \LC\ but structured by $(\X_\LP,\D_\LP, \varsTfunction{\LC},\den{\cdot}_\LC)$
where $\X_\LP$ is an infinite set of real variables and the domain $\D_\LP$ is the set of real numbers $\mathbb{R}$.
\clingoM{lp} treats either all theory atoms as  external, $\TheoryAtomsES=\TheoryAtoms$, or founded, $\TheoryAtomsES=\emptyset$.
Since it imposes no restriction on the occurrence of linear equation atoms,
the counter-intuitive behavior identified by~\citeN{jakaosscscwa17a} may emerge.
However, once we treat head atoms as founded and body atoms as external,
we obtain a one-to-one correspondence between the answers of the system and
the $\tuple{\LP,\TheoryAtomsES}$-stable models of a \clingoM{lp}-logic program $P$;
and a one-to-many correspondence between its answers and
the equilibrium models of $\htcsemanticsP(P,\LP,\TheoryAtomsES)$ for a \clingoM{lp}-logic program $P$.
 \section{Discussion}\label{sec:discussion}

Apart from the hybrid ASP approaches mentioned in the introduction, the closest work to ours
is probably \ASPAC, recently introduced by~\citeN{eitkie20a},
since it also relies on an extension of the logic \HT\ to deal with hybrid logic programs.
But while \HTC\ keeps the simple basis of propositional \HT\ by treating an external theory as a black-box
(as we have seen, the minimum requirement is using variables and denotations),
\ASPAC\ defines a complete extension of the logic itself,
lifting the formalism to first-order \HT\ with weighted formulas over semi-rings.
The main advantage of \ASPAC\ is that both the external theories and the logic programs are captured by a same homogeneous formal basis.
The price to be paid with respect to \HTC\ is a more complex semantics (logical operators become just one more type of constraints) and the requirement of a semi-ring structure.
Still, \ASPAC\ covers a wide range of constructs such as aggregates over non-Boolean variables, linear constraints or
provenance in positive datalog programs.

Our formal characterization provides several valuable contributions.
The most obvious benefit is that we have now a roadmap to follow,
not only when deciding new aspects of existing \AMT\ systems,
but also when reconsidering parts of their implementation that were originally designed with no clear hint
when facing design alternatives.
This will have an immediate impact on the next generation of existing \AMT\ systems
like \clingcon, where head atoms will start to be considered as founded, and
\clingoM{lp}, that will also accordingly introduce an implicit separation between head atoms as founded and body atoms as external.
A second important contribution is that, when proving our results,
we have tried to maintain the highest possible degree of generality in the description of the abstract theories used behind.
This paves the way for introducing new abstract theories:
We may now start classifying them in terms of the general properties,
identified in the paper (consistency, structure and denotation, closed set of external atoms, absolute complement, etc),
so that we can characterize their behavior via \HTC\ formulas.

Apart from the alignment of existing implementations and the introduction of new abstract theories,
our future work includes a comparison and investigation of other approaches, like the aforementioned~\cite{eitkie20a},
and other approaches to \AMT~\cite{barlee14b,liesus16a} or plain SMT~\cite{baseseti09a}
in view of their formal characterization in terms of \HTC.

\bibliographystyle{include/latex-class-tlp/acmtrans}

\appendix
\clearpage
\section{Proofs of results}
\label{sec:proofs}
\subsection{Proofs of Propositions~\ref{prop:lc:compositional}-\ref{prop:translation.satisfiable}}
\begin{proof}[Proof of Proposition~\ref{prop:lc:compositional}]
By construction, for every \LC-satisfiable set $S\in \mathcal{S}$, we have $\bigcap_{\thAt\in S}\den{\thAt}_\LC\neq\emptyset$.
Then, there exists a valuation $v\in\bigcap_{\thAt \in S}\den{\thAt}_\LC$ and $v\in\den{\thAt}_\LC$ for every $\thAt \in S$.
This implies that~\LC\ is compositional.
\\[10pt]
Let us show that \LC\ is consistent.
Suppose, for the sake of contradiction, that there exists a \LC-satisfiable set $S\in \mathcal{S}$ such that $\{\thAt,\comp{\thAt}\}\in S$ for a $s\in\TheoryAtoms$.
Then, since \LC\ is compositional (as shown above), there has to exists a valuation $v$, such that $v\in \den{\thAt}$ and $v\in\den{\comp{\thAt}}$.
By definition of the denotation $\den{\cdot}_\AT$, this implies that~$v$ needs to satisfy both $\sum_{1\leq i \leq n}c_i*v(x_i)\prec c_0$ and $\sum_{1\leq i \leq n}c_i*v(x_i)\comp{\prec} c_0$,
      which is a contradiction.
Therefore, every \LC-satisfiable set $S\in \mathcal{S}$ is consistent and then \LC\ is consistent.
\end{proof} \begin{proof}[Proof of Proposition~\ref{prop:entailment}]
Since \AT\ is compositional, proving $(S \cup \{\comp{\thAt}\}) \not\in \mathcal{S}$ amounts to checking $\den{S \cup \{\comp{\thAt}\}}_\AT = \emptyset$.

For the left to right direction, $S \models_\AT\ s$ is equivalent to $\den{S}_\AT \subseteq \den{\thAt}_\AT$. But then $\den{S}_\AT \cap \den{\comp{\thAt}}_\AT \subseteq \den{\thAt}_\AT \cap \den{\comp{\thAt}}_\AT = \emptyset$ and so $\den{S}_\AT \cap \den{\comp{\thAt}}_\AT = \den{S \cup \{\comp{\thAt}\}}_\AT = \emptyset$.

For the right to left direction, we proceed by contraposition. Suppose there is some $v \in \den{S}_\AT$ such that $v \not\in \den{\thAt}_\AT$.
If the complement is absolute, the latter means $v \in \den{\comp{\thAt}}_\AT$ and so $v \in \den{S}_\AT \cap \den{\comp{\thAt}}_\AT = \den{S \cup \{\comp{\thAt}\}}_\AT$ and so $\den{S \cup \{\comp{\thAt}\}}_\AT \neq \emptyset$.
\end{proof}
 \begin{proof}[Proof of Proposition~\ref{prop:complete}]
  It is enough to show that
  $$(S \cup (\comp{\TheoryAtomsES\setminus S})) \cup (\comp{\TheoryAtomsES\setminus (S \cup (\comp{\TheoryAtomsES\setminus S})}) = (S \cup (\comp{\TheoryAtomsES\setminus S}))$$
Note that $\comp{A \cup B} = \comp{A} \cup \comp{B}$ and $\comp{A \setminus B} = \comp{A} \setminus \comp{B}$.
Then,
  $$(\comp{\TheoryAtomsES\setminus (S \cup (\comp{\TheoryAtomsES\setminus S})}) =
  (\comp{\TheoryAtomsES} \setminus (\comp{S} \cup (\comp{\comp{\TheoryAtomsES}}\setminus \comp{\comp{S}}))$$
Pick~$a \in (\comp{\TheoryAtomsES} \setminus (\comp{S} \cup (\comp{\comp{\TheoryAtomsES}}\setminus \comp{\comp{S}}))$.
Then, $a \in \comp{\TheoryAtomsES}$ and~$a \notin (\comp{S} \cup (\comp{\comp{\TheoryAtomsES}}\setminus \comp{\comp{S}}))$.
The latter implies that $a \notin \comp{S}$.
Hence, $a \in \comp{\TheoryAtomsES} \setminus \comp{S} = \comp{\TheoryAtomsES \setminus S}$.

  Then, we have that $S$ is a \tsolution\ implies $S \cup (\comp{\TheoryAtomsES\setminus S})$ is \AT-satisfiable by definition.
  Due to $(S \cup (\comp{\TheoryAtomsES\setminus S})) \cup (\comp{\TheoryAtomsES\setminus (S \cup (\comp{\TheoryAtomsES\setminus S})}) = (S \cup (\comp{\TheoryAtomsES\setminus S}))$,
  we know that $S \cup (\comp{\TheoryAtomsES\setminus S})$ is \TheoryAtomsES-complete and $(S \cup (\comp{\TheoryAtomsES\setminus S})) \cup (\comp{\TheoryAtomsES\setminus (S \cup (\comp{\TheoryAtomsES\setminus S})})$ is also \AT-satisfiable,
  then $(S \cup (\comp{\TheoryAtomsES\setminus S}))$ is a \TheoryAtomsES-complete \tsolution.

  Conversely, if $S$ is a \TheoryAtomsES-complete \tsolution,
  then $S \cup (\comp{\TheoryAtomsES\setminus S})=S$ is \AT-satisfiable by definition of \TheoryAtomsES-complete and \tsolution.
\end{proof}
 \begin{proof}[Proof of Proposition~\ref{prop:consistentclosed}]
Suppose, for the sake of contradiction, that there exists an incomplete \tsolution~$S$.
Then, $S\neq S\cup\comp{(\TheoryAtomsES\setminus S)}$ and, thus, there exists a theory atom~$\thAt\in \comp{(\TheoryAtomsES\setminus S)}$ with $\thAt\in\comp{\TheoryAtomsES}$, $\thAt \notin\comp{S}$ and $\thAt \notin S$.
Furthermore,
$\thAt\in\comp{\TheoryAtomsES}$ implies $\{\thAt,\comp{\thAt}\}\subseteq\TheoryAtomsES$ since \TheoryAtomsES\ is closed.
Then, $\comp{\thAt}\in \comp{(\TheoryAtomsES\setminus S)}$ since $\thAt \notin S$ and $\thAt\in\TheoryAtomsES$.
This leads to a contradiction because $\{\thAt,\comp{\thAt}\}\subseteq \comp{(\TheoryAtomsES\setminus S)}$,
and therefore $S\cup\comp{(\TheoryAtomsES\setminus S)}$ is not \AT-satisfiable.
Since \AT\ is consistent, this implies that $S$ not a \tsolution, which is a contradiction with the assumption that~$S$ is a \tsolution.
Consequently, all \tsolution{s} $S$ are \TheoryAtomsES-complete for consistent abstract theory \AT\ and closed \TheoryAtomsES.
\end{proof}
 \begin{proof}[Proof of Proposition~\ref{prop:translation.satisfiable}]
    Since~$\AT$ is compositional, set~$S$ is \mbox{$\AT$-satisfiable}
    iff there is a valuation~${w: \X_\AT \rightarrow \D_\AT}$ such that~$w \in \den{\thAt}_\AT$ for all~$\thAt \in S$.
    \begin{itemize}
      \item Assume first that $S$ is \mbox{$\AT$-satisfiable}
      and, thus, that there is a valuation~${w: \X_\AT \rightarrow \D_\AT}$ such that~$w \in \den{\thAt}_\AT$ for all~$\thAt \in S$.
Let~$v : \X \rightarrow \D_\undefined$ be any valuation such that~$v(x) = w(x)$ for all~$x \in \X_\AT$.
Then, $w \in \den{ \thAt}_\AT$ implies $v \in \den{ \tau(\thAt)}$ and, therefore, we get that $\tau(S)$ is satisfiable in~$\HTC$.

    \item Conversely, if $\tau(S)$ is satisfiable in~$\HTC$, then there is a valuation~$v$ that satisfies~$v \in \den{ \tau(\thAt)}$ for all~$\thAt \in S$.
By construction, this implies that there is some valuation~${w: \X_\AT \rightarrow \D_\AT}$ such that~$w \in \den{ \thAt}_\AT$ for all~$\thAt \in S$.
Consequently, $S$ is \mbox{$\AT$-satisfiable}.\qed
    \end{itemize}
    \let\proofbox\relax
  \end{proof} \subsection{Proofs of the Main Theorem}

To pave the way between the transformation approach described in Section~\ref{sec:translation:revisited} and the translation in~$\HTC$ described in Section~\ref{sec:htcsemantics},
we introduce a second translation into~$\HTC$ that is close to the transformation approach.
The proof of the Main Theorem is then divided into two main lemmas: the first establishes the correspondence between the transformation approach described in Section~\ref{sec:translation:revisited} and this second translation, and the second establishes the correspondence in~$\HTC$ between both translations.

Let us start by describing this second translation that we denote~$\htcsemanticsP_2$.
The most relevant feature of this translation is that it decouples the generation of possible abstract theory solutions $S \subseteq \TheoryAtoms$ from the derivation of their atoms $\thAt \in S$ in the logic program.
In particular, rather than directly including constraints $\htcsemanticsA(\thAt)$ in the translation of program rules as done with \eqref{eq:htcsemantics:rule}, we use now a new, auxiliary propositional atom $\prop{\thAt}$ whose connection to the constraint $\htcsemanticsA(\thAt)$ is be explicitly specified by using additional formulas.
In that way, the rules of the \tprogram\ $P$ will correspond now to an \HTC\ encoding of a regular,  propositional logic program we call~$p(P)$.

Translation $\htcsemanticsP_2$ produces an \HTC-theory with signature $\tuple{\X_2,\D_2,\C_2}$ and denotation~$\den{\cdot}_2$ that extends the signature $\tuple{\X,\D,\C}$ and denotation $\den{\cdot}$ of $\htcsemanticsP$ as follows:
\begin{align}
\X_2 \ \ &= \ \ \X \cup \{ \, \prop{\thAt} \mid \thAt \in \TheoryAtoms \, \}\\
\D_2 \ \ &= \ \ \D\\
\C_2 \ \ &= \ \ \C \ \cup \{ \, ``\prop{\thAt}=\mathbf{t}\text{''} \mid \thAt \in \TheoryAtoms \, \}
\end{align}
that is, we extend the set of variables $\X$ with one fresh variable $\prop{\thAt}$ for each abstract theory atom $\thAt \in \TheoryAtoms$ and the constraints $\C$ with the corresponding propositional constraint atom $``\prop{\thAt}=\mathbf{t}\text{''}$.
The denotation $\den{\cdot}$ for the $\htcsemanticsP_2$ translation simply extends the denotation for $\htcsemanticsP$ by including the already seen fixed denotation for propositional atoms applied to the new constraints $``\prop{\thAt}=\mathbf{t}\text{''}$.

The new \HTC-encoding $\htcsemanticsP_2(P,\AT,\TheoryAtomsES)$ is defined for a same \tprogram\ $P$ over $\langle\Atoms,\TheoryAtoms,\TheoryAtomsES\rangle$ as before, but consists of three sets of formulas:
\begin{eqnarray}
\htcsemanticsP_2(P,\AT,\TheoryAtomsES) \eqdef \htctsols(\AT,\TheoryAtomsES) \cup p(P) \cup \bridge(P,\AT,\TheoryAtomsES) \label{f:tau2}
\end{eqnarray}
where, informally speaking: $\htctsols(\AT,\TheoryAtomsES)$ generates arbitrary sets of abstract theory solutions in terms of constraints $\htcsemanticsA(\thAt)$; $p(P)$ corresponds to an \HTC-encoding of a propositional program for atoms $\prop{\thAt}$; and $\bridge(P,\AT,\TheoryAtomsES)$ fixes the connection between each atom $\prop{\thAt}$ and its corresponding constraint $\htcsemanticsA(\thAt)$.

We describe now each one of these three sets.
Theory~$\htctsols(\AT,\TheoryAtomsES)$ consists of formulas
      \begin{align}
            \htcsemanticsA(\thAt) &\vee \neg \htcsemanticsA(\thAt) \hspace{1.5cm}&&\text{for every theory atom } \thAt \in \TheoryAtoms
            \label{eq:guess:choice}
            \\
            \neg \htcsemanticsA(\thAt) &\to \, \htcsemanticsA(\comp{\thAt}) &&\text{for every theory atom } \thAt \in \TheoryAtomsES
            \label{eq:guess:strict}
      \end{align}
Here, the truth of constraint atom $\htcsemanticsA(\thAt)$ for $\thAt\in\TheoryAtoms$ describes the inclusion of $\thAt$ in a possible \mbox{$\langle\AT,\TheoryAtomsES\rangle$-solution} $S$.
If $\htcsemanticsA(\thAt)$ is true or false, it represents that $\thAt\in S$ or $\thAt \not\in S$, respectively.
The choice in~\eqref{eq:guess:choice} generates all possible sets $S$.
Implication~\eqref{eq:guess:strict} enforces that, whenever an external theory atom is \emph{not} included in the set $\thAt \not\in S$, its complement must hold $\comp{\thAt} \in S$.
Note that no atom of the form~$\prop{\thAt}$ occurs in~$\htctsols(\AT,\TheoryAtomsES)$.
Therefore, any equilibrium model~$\tuple{t,t}$ of~$\htctsols(\AT,\TheoryAtomsES)$ satisfies~$t(\prop{\thAt}) = \mathbf{u}$.

Back to our example, given the \LC\ atoms $\TheoryAtoms=\{\thAt_1,\thAt_2,\thAt_3,\thAt_4\}$ with $\thAt_1=(\sumxy)$, $\thAt_2=(\sumyz)$, $\thAt_3=(\sumxyc)$, $\thAt_4=(\sumyzc)$ seen before, where $\thAt_1$ and $\thAt_3$ are external,  and given the $\htcsemanticsA(\thAt)$ translation for \LC\ atoms, the theory $\htctsols(\LC,\TheoryAtomsES)$ would amount to:
\begin{eqnarray}
x+y=4 \vee \neg(x+y =4) \label{f:phi1}\\
y+z=2 \vee \neg(y+z =2) \\
x+y\neq 4 \vee \neg(x+y \neq 4) \\
y+z \neq 2 \vee \neg(y+z \neq 2) \\
\neg (x+y=4) \to (x+y \neq 4) \\
\neg (x+y\neq 4) \to (x+y=4) \label{f:phi6}
\end{eqnarray}
Notice the difference between the explicit difference in atoms $F \neq 4$ and the default negation of an equality in literals $\neg (F=4)$.
The former hold when all variables in $F$ are defined, but $F$ is not $4$, whereas the latter still hold even when some variable in $F$ is undefined.

For brevity, we may refer in the following to an equilibrium model $\tuple{t,t}$ of a theory $\Gamma$ instead as a stable model $t$ of theory $\Gamma$,
and say $t\models \Gamma$ instead of $\tuple{t,t}\models\Gamma$ for total interpretation $\tuple{t,t}$.

The behavior of $\htctsols(\AT,\TheoryAtomsES)$ is formally characterized by the following result:
\begin{proposition}\label{prop:htc.tsols}
Given a consistent, compositional abstract theory $\AT=\langle \TheoryAtoms, \mathcal{S}, \comp{\cdot} \,\rangle$ and a closed set~$\TheoryAtomsES\subseteq\TheoryAtoms$ of external theory atoms, there is a one-to-many correspondence between the set of complete \mbox{$\langle\AT,\TheoryAtomsES\rangle$-solutions} and the stable models of~$\htctsols(\AT,\TheoryAtomsES)$ such that
\begin{itemize}
\item if $t$ is a stable model of~$\htctsols(\AT,\TheoryAtomsES)$, then $\{\thAt \in \TheoryAtoms \mid t \models \htcsemanticsA(\thAt)\}$ is a complete \mbox{$\langle\AT,\TheoryAtomsES\rangle$-solution};\label{thm:htctsols:eqtos}

\item if $S$ is a complete \mbox{$\langle\AT,\TheoryAtomsES\rangle$-solution}, then there exists a stable model $t$ of theory $\htctsols(\AT,\TheoryAtomsES)$ such that $S =\{\thAt \in \TheoryAtoms \mid t \models \htcsemanticsA(\thAt)\}$.
\label{thm:htctsols:stoeq}
\end{itemize}
\end{proposition}
Note that the correspondence is not one-to-one because, for a same \mbox{$\langle\AT,\TheoryAtomsES\rangle$-solution} $S$, we may get several (even infinite) variable assignments that satisfy the same constraint atoms.
For instance, for a \mbox{$\langle\AT,\TheoryAtomsES\rangle$-solution} just consisting of linear constraint atom~$\code{\&sum\{x\} > 3}$, we will get one stable model $t$ with $t(\code{x})=n$ for each integer $n \in \mathbb{Z}$ strictly greater than $3$.

As we can see \eqref{f:tau2}, the second group of formulas in $\htcsemanticsP_2(P,\AT,\TheoryAtomsES)$ is $p(P)$ which, as we did with~\eqref{eq:htcsemantics:rule}, will perform an atom-level translation to all rules in $P$.
In this case, the novelty comes in the translation for theory atoms $\thAt \in \TheoryAtoms$, defined as  $\htcsemanticsA_2(\thAt) \eqdef \prop{\thAt}$, while regular atoms and $\bot$ are translated as before, namely, $\htcsemanticsA_2(\regAt) \eqdef \htcsemanticsA(\regAt) = \prop{\regAt}$ for $\regAt \in \Atoms$ and $\htcsemanticsA_2(\bot) \eqdef \bot$.
Then, $p(P)$ is the result of replacing the atom translation $\htcsemanticsA$ by $\htcsemanticsA_2$ in the formula~\eqref{eq:htcsemantics:rule} for each rule \eqref{theory:rule} in $P$.
As an illustration, given our \LC\ example with atoms $\TheoryAtoms=\{\thAt_1,\thAt_2,\thAt_3,\thAt_4\}$ suppose, for the sake of readability, we just write $\htcsemanticsA(\thAt_i)=\prop{i}$ for $i=1\dots 4$.
Then $p(P_{(\ref{ex:clingcon:rule:one}/\ref{ex:clingcon:rule:two})})$ would just consist of the two implications:
\begin{align}
\prop{1} & \to  \prop{\regAt} \label{f:pP1}\\
\prop{\regAt} & \to  \prop{2} \label{f:pP2}
\end{align}

The third group of formulas in $\htcsemanticsP_2(P,\AT,\TheoryAtomsES)$ is the theory $\bridge(P,\AT,\TheoryAtomsES)$ defined as:
\begin{align}
\htcsemanticsA(\thAt) & \to  \prop{\thAt} && \text{for each } \thAt \in \TheoryAtomsES \label{eq:htctrans:external}\\
\neg \htcsemanticsA(\thAt) \wedge \prop{\thAt} & \to  \bot && \text{for each } \thAt \in \TheoryAtoms \setminus \TheoryAtomsES \label{eq:htctrans:founded}
\end{align}
As we can see, the implication for external atoms \eqref{eq:htctrans:external} is quite straightforward: it forces the propositional atom $\prop{\thAt}$ to be true when the corresponding constraint holds.
In the case of founded theory atoms, $\TheoryAtoms \setminus \TheoryAtomsES$, our \HTC\ theory may not provide justification but their truth has to be aligned with the truth of the constraint in a \tsolution.
This is handled via the integrity constraints in \eqref{eq:htctrans:founded}, forbidding $\prop{\thAt}$ to hold if $\htcsemanticsA(\thAt)$ is false.
In our \LC-program example, $\bridge(P_{(\ref{ex:clingcon:rule:one}/\ref{ex:clingcon:rule:two})},\LC,\TheoryAtomsES)$ would correspond to:
\begin{align}
x+y=4 & \to  \prop{1} \label{f:bridge1}\\
x+y \neq 4 & \to  \prop{3}\\
\neg (y+z=2) \wedge \prop{2} & \to  \bot \\
\neg (y+z \neq 2) \wedge \prop{4} & \to  \bot \label{f:bridge4}
\end{align}
so that, our final translation $\htcsemanticsP_2(P_{(\ref{ex:clingcon:rule:one}/\ref{ex:clingcon:rule:two})},\LC,\TheoryAtomsES)$ amounts to the set of formulas $\{\eqref{f:phi1}-\eqref{f:phi6}\} \cup \{\eqref{f:pP1}, \eqref{f:pP2}\} \cup \{\eqref{f:bridge1}-\eqref{f:bridge4}\}$.
As we can see, this is much more verbose than $\htcsemanticsP(P_{(\ref{ex:clingcon:rule:one}/\ref{ex:clingcon:rule:two})},\LC,\TheoryAtomsES)=\{\eqref{f:tauP1}-\eqref{f:choice42}\}$, but has the advantage of being structurally closer to the program transformation in~\eqref{eq:transformation}, so its correctness has a more direct proof.

As we mentioned above, $\htctsols(\AT,\TheoryAtomsES)$ generates arbitrary sets of abstract theory solutions in terms constraints $\htcsemanticsA(\thAt)$ that is bridge with~$p(P)$ to produce the stable models.
This separation can be made precise using the splitting result by~\cite[Proposition~12]{cafascwa20b}.
Proposition~\ref{prop:second_translation.splitting} below makes this intuition precise.
To this end, we need the following notation.
We define $\htctrans(P,\AT,\TheoryAtomsES) = p(P) \cup \bridge(P,\AT,\TheoryAtomsES)$, that is, the  two sets of formulas in~\eqref{f:tau2} dealing with propositional variables $\prop{\thAt}$.
Given a fixed valuation $t$, we further define $\htctrans(P,\AT,\TheoryAtomsES,t)$ as the result of replacing in $\Gamma(P,\AT,\TheoryAtomsES)$ every constraint atom~$c \in \C$ by~$\top$ if $t \in \den{c}$ and~$\bot$ otherwise.
Let us further denote $p(\Atoms \cup \TheoryAtoms) \subseteq \X_2$ as the set of all auxiliary propositional variables: $p(\Atoms \cup \TheoryAtoms) \eqdef \{\prop{\regAt} \mid \regAt \in \Atoms\}\cup \{\prop{\thAt} \mid \thAt \in \TheoryAtoms\}$.
Note that $\C \subseteq \C_2$ refers to the constraint atoms in $\htcsemanticsA$, and so, $\C \cap p(\Atoms \cup \TheoryAtoms)=\emptyset$.

\begin{samepage}
\begin{proposition}{\label{prop:second_translation.splitting}}
    Given a \tprogram\ $P$ over $\langle\Atoms,\TheoryAtoms,\TheoryAtomsES\rangle$ such that \TheoryAtomsES\ is a closed;
    a consistent, compositional abstract theory $\AT=\langle \TheoryAtoms, \mathcal{S}, \comp{\cdot} \,\rangle$;
    $t$ is a stable model of~$\htcsemanticsP_2(P,\AT,\TheoryAtomsES)$ iff
    $\restr{t}{\X}$ is a stable model of $\htctsols(\AT,\TheoryAtomsES)$ and
    $\restr{t}{\; p(\Atoms\cup\TheoryAtoms)}$ is a stable model of
    $\htctrans(P,\AT,\TheoryAtomsES,t)$.
\end{proposition}
\end{samepage}

The following two results make the needed connections between this translation, the program transformation and the other translation.

\begin{proposition}\label{prop:second_translation}
Given a \tprogram\ $P$ over $\langle\Atoms,\TheoryAtoms,\TheoryAtomsES\rangle$ such that \TheoryAtomsES\ is closed;
a consistent, compositional abstract theory $\AT=\langle \TheoryAtoms, \mathcal{S}, \comp{\cdot} \,\rangle$,
there is a one-to-many correspondence between the \mbox{$\langle\AT,\TheoryAtomsES\rangle$-stable} models of~$P$ and the stable models of theory~$\htcsemanticsP_2(P,\AT,\TheoryAtomsES)$ such that
\begin{itemize}
\item if $v$ is an stable model of~$\htcsemanticsP_2(P,\AT,\TheoryAtomsES)$,
then~$\{\thAt\in \Atoms\cup\TheoryAtoms \mid v \in \den{\prop{\thAt}}\}$ is a \mbox{$\langle\AT,\TheoryAtomsES\rangle$-stable} model of $P$.

\item if $X$ is a \mbox{$\langle\AT,\TheoryAtomsES\rangle$-stable} model of $P$, then there exists a stable model $v$ of $\htcsemanticsP_2(P,\AT,\TheoryAtomsES)$
such that $\restr{v}{p(\Atoms\cup\TheoryAtoms)}=\{(\prop{\anyAt},\mathbf{t}) \mid \anyAt \in X\}$,
\end{itemize}
\end{proposition}

\begin{proposition}\label{prop:translations}
Given a \tprogram\ $P$ over $\langle\Atoms,\TheoryAtoms,\TheoryAtomsES\rangle$ such that \TheoryAtomsES\ is a closed; a consistent, compositional abstract theory $\AT=\langle \TheoryAtoms, \mathcal{S}, \comp{\cdot} \,\rangle$;
if~$t$ and~$v$ are two valuations such that
\begin{align*}
      v \ \ = \ \ t &\cup \{(\prop{\thAt},\mathbf{t}) \mid \thAt\in \TheoryAtomsES \text{ and } t \in \den{\htcsemanticsA(\thAt)} \}
      \\
                      &\cup \{(\prop{\thAt},\mathbf{t}) \mid \thAt\in \TheoryAtomsDN, \ (B \to \thAt) \in P    \text{ and } t \models \htcsemanticsA(B) \}
\end{align*}
then $t$ is a stable model of theory~$\htcsemanticsP(P,\AT,\TheoryAtomsES)$ iff~$v$ is a stable model of theory~$\htcsemanticsP_2(P,\AT,\TheoryAtomsES)$.
\end{proposition}

\begin{proof}[Proof of the Main Theorem]
Assume that~$t$ is a stable model of~$\htcsemanticsP(P,\AT,\TheoryAtomsES)$.
Then, from Proposition~\ref{prop:translations},
$v$ is a stable model of~$\htcsemanticsP_2(P,\AT,\TheoryAtomsES)$.
From Proposition~\ref{prop:second_translation}, this implies that~$X = \{\anyAt\in \Atoms\cup\TheoryAtoms \mid v \in \den{\prop{\anyAt}}\}$ is a \mbox{$\langle\AT,\TheoryAtomsES\rangle$-stable} model of $P$.
Note that, by construction,
$v \in \den{\prop{\anyAt}}$ iff~$t \in \den{\htcsemanticsA(\anyAt)}$ for all~$\anyAt \in \Atoms \cup \TheoryAtomsES$.
Hence, \eqref{eq:1:thm:first_translation} holds.
\\[10pt]
The other way around.
Assume that~$X$ is a \mbox{$\langle\AT,\TheoryAtomsES\rangle$-stable} model of~$P$.
Then, from Proposition~\ref{prop:second_translation},
there exists a stable model $v$ of $\htcsemanticsP_2(P,\AT,\TheoryAtomsES)$
such that $\restr{v}{p(\Atoms\cup\TheoryAtoms)}=\{(\prop{\anyAt},\mathbf{t}) \mid \anyAt \in X\}$.
Hence, $X = \{ \, \anyAt \in \Atoms \cup \TheoryAtoms \,\mid \, v \in \den{\prop{\anyAt}} \, \}$.
Furthermore, from Proposition~\ref{prop:translations},
this also implies that there there is a stable model~$t$ of~$\htcsemanticsP(P,\AT,\TheoryAtomsES)$ such that
\begin{align*}
      v \ \ = \ \ t &\cup \{(\prop{\thAt},\mathbf{t}) \mid \thAt\in \TheoryAtomsES \text{ and } t \in \den{\htcsemanticsA(\thAt)} \}
      \\
                      &\cup \{(\prop{\thAt},\mathbf{t}) \mid \thAt\in \TheoryAtomsDN, \ (B \to \thAt) \in P    \text{ and } t \models \htcsemanticsA(B) \}
\end{align*}
Hence, \eqref{eq:1:thm:first_translation} holds.
\end{proof}

The rest of this section is devoted to prove Propositions~\ref{prop:htc.tsols}-\ref{prop:translations}.
\subsubsection*{Proof of Propositions~\ref{prop:htc.tsols}-\ref{prop:second_translation.splitting}}
\begin{proof}[Proof of Proposition~\ref{prop:htc.tsols}]
\emph{First statement}.
Assume that~$t$ be a stable model of~$\htctsols(\AT,\TheoryAtomsES)$ and let
$$S=\{\thAt\in \TheoryAtoms \mid t \in \den{ \tau(\thAt) }\}$$
be a set of theory atoms.
By construction, $\tau(S)$ is satisfiable in~$\HTC$ and, from Proposition~\ref{prop:translation.satisfiable}, this implies that~$S$ is \mbox{\AT-satisfiable}.
Hence, it only remains to be shown that~$S$ is complete.
Note that, since~$t$ be a stable model of~$\htctsols(\AT,\TheoryAtomsES)$, it follows that~$v$ satsifes all implications of the form of~\eqref{eq:guess:strict}.
Therefore, we get that every $\thAt\in\TheoryAtomsES$ satisfies that~$v\notin\den{\tau(\thAt)}$ implies $v\in\den{\tau(\comp{\thAt})}$.
Then, we get
\begin{align*}
      \comp{\TheoryAtomsES \setminus S}
      &=\{\comp{\thAt} \mid \thAt \in\TheoryAtomsES \text{ and } \thAt \notin S \}
      \\
      &= \{\comp{\thAt} \mid \thAt \in\TheoryAtomsES \text{ and } t
      \notin \den{ \tau(\thAt) } \}
      \\
      &= \{\comp{\thAt} \mid \thAt \in\TheoryAtomsES \text{ and }
      t \in \den{\tau(\comp{\thAt})} \} \subseteq S
\end{align*}
Consequently, $S = S \cup \comp{\TheoryAtomsES \setminus S}$ and, since it is \AT-satisfiable and complete, it is a complete \AT-solution.
\\[10pt]
\emph{Second statement}.
Assume that $S$ is a complete \mbox{\AT-solution}.
Then,
$S$ is \AT-satisfiable (Proposition~\ref{prop:complete})
and, thus, $\tau(S)$ is satisfiable in~$\HTC$ (Proposition~\ref{prop:translation.satisfiable}).
This implies that there is a valuations~$w: \X \longrightarrow \D_\undefined$ such that~$w\in \den{ \tau( \thAt)}$ for all~$\thAt \in S$.
Let $t$ be a valuation such that
\begin{itemize}
      \item $t(x) = w(x)$ if $x \in \varsAT{\thAt}$ for some~$\thAt \in S$,
      \item $t(x) = \undefined$ otherwise.
\end{itemize}
Then, $v\in\den{\tau(\thAt)}$ holds for every~$\thAt\in S$.
Note that $w\in\den{\thAt}_\AT$ and $\restr{t}{\varsAT{\thAt}} = \restr{w}{\varsAT{\thAt}}$ for all $\thAt\in S$.
Therefore, $t$ satisfies all formulas of the form of~\eqref{eq:guess:choice} and it remains to be shown that~$t$ also satisfies all formulas of the form of~\eqref{eq:guess:strict}.
For that pick any theory atom~$\thAt \in \TheoryAtomsES \setminus S$.
\begin{itemize}
      \item Since~$S$ is complete, we get that $\comp{\thAt} \in S$.
      \item This implies that, $t \in \den{\tau(\comp{\thAt})}$.
\end{itemize}
Therefore, $v$ satisfies~$\htctsols(\AT,\TheoryAtomsES)$.

Furthermore, for any valuation $h \subset t$, there is some theory atom~$\thAt \in S$ and variable~$x  \in \varsAT{\thAt}$ such that~$h(x) = \undefined$ and, thus, $h \notin \den{\tau(\thAt)}$.
This implies~$\tuple{h,t} \not\models \tau(\thAt) \vee \neg \tau(\thAt)$ (Proposition~3 by~\citeNP{cakaossc16a}) and, thus, that~$\tuple{t,t}$ is an equilibrium model of~$\htctsols(\AT,\TheoryAtomsES)$.
\end{proof} \begin{proof}[Proof of Proposition~\ref{prop:second_translation.splitting}]
    Set~$\X$ is a splitting set of~$\htcsemanticsP_2(P,\AT,\TheoryAtomsES)$
    in the sense of Definition~10 by~\citeN{cafascwa20b}.
Note that every rule~$r\in \htctrans(P,\AT,\TheoryAtomsES)$ satisfies~$\vars{\head{r}} \subseteq p(\Atoms \cup \TheoryAtoms)$ and, thus, we get
    \begin{align*}
          \vars{\head{r}}\cap \X =\emptyset \text{ for every } r\in \htctrans(P,\AT,\TheoryAtomsES)
    \end{align*}
    Then, from Proposition~12 by~\citeN{cafascwa20b}, we get that~$v$ is a stable model of~$\htcsemanticsP_2(P,\AT,\TheoryAtomsES)$ iff
    $\restr{v}{\X}$ is a stable model of $\htctsols(\AT,\TheoryAtomsES)$ and
    $\restr{v}{\; p(\Atoms\cup\TheoryAtoms)}$ is a stable model of the theory
    obtained from~$\htctrans(P,\AT,\TheoryAtomsES)$ by replacing all variables in~$\X$ by its value in~$v$.
After trivial simplifications this amounts to~$\htctrans(P,\AT,\TheoryAtomsES,v)$.
\end{proof}  \subsubsection*{Proof of Proposition~\ref{prop:second_translation}}

\begin{lemma}\label{lem:aux:thm:second_translation}
Let $\AT=\langle \TheoryAtoms, \mathcal{S}, \comp{\cdot} \,\rangle$ be a complement-consistent, compositional abstract theory,
let~$P$ be a \tprogram\ over~$\langle\Atoms,\TheoryAtoms,\TheoryAtomsES\rangle$ such that~$\TheoryAtomsES$ is closed.
Let~$t : \X_2 \longrightarrow \D_\undefined$ be a valuation such that~$S = \{\thAt\in \TheoryAtoms \mid v \in \denn{\htcsemanticsA(\thAt)} \}$ is a complete {$\langle\AT,\TheoryAtomsES\rangle$-solution}.
Then,
${X = \{\anyAt\in \Atoms \cup \TheoryAtoms \mid t \in \denn{p(\anyAt)} \}}$
is a \mbox{$\langle\AT,\TheoryAtomsES\rangle$-stable} model of~$P$
iff~$\restr{t}{\; p(\Atoms\cup\TheoryAtoms)}$
is a stable model of~$\htctrans(P,\AT,\TheoryAtomsES,t)$.
\end{lemma}

\begin{proof}
From Proposition~2 by~\citeN{cakaossc16a} and Proposition~2 by~\citeN{pearce96a},
set~$X$ is a stable model of program~\eqref{eq:transformation}
iff~$\restr{t}{\; p(\Atoms\cup\TheoryAtoms)}$ is stable model of theory
\begin{align}
      p(P)\cup
      \{p(\thAt) \mid \thAt \in
      (\TheoryAtomsES \cap S) \} \ \cup
      \{ p(\thAt) \to \bot \mid \thAt \in
      (\TheoryAtomsDN \setminus S) \}
      \label{eq:1:lem:aux:thm:second_translation}
\end{align}
Furthermore, it is easy to see that
\begin{gather*}
      X = \{\anyAt\in \Atoms \cup \TheoryAtoms \mid v \in \denn{p(\anyAt)} \}
        = \{\anyAt\in \Atoms \cup \TheoryAtoms \mid \restr{t}{\; p(\Atoms\cup\TheoryAtoms)} \in \denn{p(\anyAt)} \}
\end{gather*}
Hence, it is enough to show that~$\htctrans(P,\AT,\TheoryAtomsES,t)$ and~\eqref{eq:1:lem:aux:thm:second_translation} have the same stable models.
\\[5pt]
For this, note that rules in~$\htctrans(P,\AT,\TheoryAtomsES,t)$ can be separated in the following groups:
\begin{align*}
      &p(P)\cup \\
      &\{\phantom{\neg}\top\to p(\thAt) \mid \thAt \in \TheoryAtomsES \text{ and } t\in \denn{\tau(\thAt)}\} \ \cup
      \\
      &\{\phantom{\neg}\bot\to p(\thAt) \mid \thAt \in \TheoryAtomsES \text{ and } t\notin \denn{\tau(\thAt)}\} \} \ \cup
      \\
      &\{\neg \top \wedge p(\thAt) \to \bot \mid \thAt \in \TheoryAtomsDN \text{ and } t\in \denn{\tau(\thAt)}\}\}
      \\
      &\{\neg \bot \wedge p(\thAt) \to \bot \mid \thAt \in \TheoryAtomsDN \text{ and } t\notin \denn{\tau(\thAt)}\}
\end{align*}
which after some trivial simplifications amounts to the the strongly equivalent theory
\begin{align*}
      p(P)\cup
      \{p(\thAt) \mid \thAt \in \TheoryAtomsES \text{ and } t\in \denn{\tau(\thAt)}\} \cup
      \{p(\thAt) \to \bot \mid \thAt \in \TheoryAtomsDN \text{ and } t\notin \denn{\tau(\thAt)}\}
\end{align*}
Since $S = \{\thAt\in \TheoryAtoms \mid t \in \denn{\htcsemanticsA(\thAt)} \}$,
it is easy to see that this is the same as~\eqref{eq:1:lem:aux:thm:second_translation}.
\end{proof}

\begin{proof}[Proof of Proposition~\ref{prop:second_translation}.]
\emph{First statement}.
Assume that $t$ is an stable model of~$\htcsemanticsP_2(P,\AT,\TheoryAtomsES)$
and let
$${X = \{\thAt\in \Atoms \cup \TheoryAtoms \mid t \in \denn{p(\thAt)} \}}$$
From Proposition~\ref{prop:second_translation.splitting}, this implies that $\restr{t}{\X}$ is a stable model of $\htctsols(\AT,\TheoryAtomsES)$ and that
$\restr{t}{\; p(\Atoms\cup\TheoryAtoms)}$ is a stable model of
$\htctrans(P,\AT,\TheoryAtomsES,t)$.
From Proposition~\ref{prop:htc.tsols}, the former implies that set $S = \{\thAt\in \TheoryAtoms \mid t \in \denn{\htcsemanticsA(\thAt)} \}$ is a complete \mbox{$\langle\AT,\TheoryAtomsES\rangle$-solution}.
From Lemma~\ref{lem:aux:thm:second_translation}, this plus the the fact that
$\restr{t}{\; p(\Atoms\cup\TheoryAtoms)}$ is a stable model of
$\htctrans(P,\AT,\TheoryAtomsES,t)$
implies that~$X$ is a \mbox{$\langle\AT,\TheoryAtomsES\rangle$-stable} model of~$P$.
\\[10pt]
\emph{Second statement}.
Assume that $X$ is a \mbox{$\langle\AT,\TheoryAtomsES\rangle$-stable} model of $P$.
Then, there exists a complete \mbox{$\langle\AT,\TheoryAtomsES\rangle$-solution} $S$,
      such that $X$ is a stable model of program transformation~\eqref{eq:transformation} of $P$ wrt $S$.
From Proposition~\ref{prop:htc.tsols}, this implies that
      there exists a stable model $w$ of $\htctsols(\AT,\TheoryAtomsES)$
      such that $S=\{ \thAt \in \TheoryAtoms \mid w \in \denn{ \tau(\thAt) }\}$.

Let~$t: \X_2 \longrightarrow \D_\undefined$ be a valuation such that
\begin{align*}
      t(x) = \begin{cases}
            w(x) &\text{if } x \in \X
            \\
            \mathbf{t} &\text{if } x = p_\thAt \text{ for some } \thAt \in \Atoms \cup \TheoryAtoms \text{ and } \thAt \in X
            \\
            \undefined &\text{if } x = p_\thAt \text{ for some } \thAt \in \Atoms \cup \TheoryAtoms \text{ and } \thAt \notin X
      \end{cases}
\end{align*}
Then,
${X = \{\thAt\in \Atoms \cup \TheoryAtoms \mid t \in \denn{p(\thAt)} \}}$
and, from Lemma~\ref{lem:aux:thm:second_translation}, we get that~$\restr{t}{\; p(\Atoms\cup\TheoryAtoms)}$
is a stable model of~$\htctrans(P,\AT,\TheoryAtomsES,t)$.
Note that, by construction, $\restr{t}{\X} = \restr{w}{\X}$.
From Proposition~\ref{prop:second_translation.splitting}, these facts imply that~$t$ is a stable model of~$\htcsemanticsP_2(P,\AT,\TheoryAtomsES)$.
\end{proof} \subsubsection*{Proof of Proposition~\ref{prop:translations}}

\begin{lemma}[Completion of definitions]\label{lem:completion.of.definitions}
Let~$\Gamma$ be a set of $\HTC$-rules and let~$C$ be a set of constraint atom such that,
for every~$c \in C$, there is some variable~$x_c \in \vars{c}$ that does not occur in a head of~$\Gamma$ nor in any other atom in~$C$.
For each~$c \in C$, let~$F_c$ be a a conjunction of literals that does not contain any variables occurring in atoms from~$C$.
Then, theories~$\Gamma \cup \{ F_c \to c \mid c \in C \}$ and~$\Gamma \cup \{ F_c \leftrightarrow c  \mid c \in C \}$ have the same stable models.
\end{lemma}

\begin{proof}
Let~$\Gamma_I = \Gamma \cup \{ F_c \to c \mid c \in C \}$
and~$\Gamma_D = \Gamma \cup \{ F_c \leftrightarrow c  \mid c \in C \}$.
Then, $\tuple{h,t} \models \Gamma_D$ implies~$\tuple{h,t} \models \Gamma_I$ for every interpretation~$\tuple{h,t}$.
\\[10pt]
Assume first that~$t$ is a stable model of~$\Gamma_I$.
From the above observation, to show that~$t$ is a stable model of~$\Gamma_D$, it is enough to prove that~$t$ is a model of~$\Gamma_D$.
Then, we need to show that~$t \models c \to F_c$ for an arbitrary~$c \in C$.
Assume that~$t \models c$.
Then, $v(x_c)\neq\undefined$.
From Proposition~11 by~\citeN{cafascwa20b} and the fact that~$F_c \to c$ is the only rule in~$\Gamma_I$ where~$x_c$ occurs in the head, this implies that~$t \models F_c$ and, therefore, we get that~$t \models c \to F_c$.
Consequently, $t$ is a stable model of~$\Gamma_D$.
\\[10pt]
Assume now that~$t$ is a stable model of $\Gamma_D$.
Then, $t$ is a model of~$\Gamma_I$ and to show that it is also an stable model we need to show that~$\tuple{h,t} \not\models \Gamma_I$ for any~$h \subset t$.
Suppose, for the sake of contradiction, that~$\tuple{h,t} \models \Gamma_I$ holds for some~$h \subset t$.
Let~$h'$ be a valuation such that
\begin{align*}
    h'(x) = \begin{cases}
        \undefined &\text{if $x = x_c$ with~$c \in C$ and $\tuple{h,t} \not\models F_c$ }
        \\
        t(x) &\text{if $x = x_c$ with~$c \in C$ and $\tuple{h,t} \models F_c$ }
        \\
        h(x) &\text{otherwise}
    \end{cases}
\end{align*}
By construction, $\tuple{h',t} \models F_c \leftrightarrow c$ for all~$c \in C$ because~$t \models F_c \leftrightarrow c$.
Furthermore, $\tuple{h,t} \models \Gamma_I$ implies~$\tuple{h,t} \models \Gamma$ and this implies~$\tuple{h',t} \models \Gamma$.
Note that~$h'$ only undefines variables that cannot occur in head of~$\Gamma$.
Hence, $\tuple{h',t} \models \Gamma_D$ and, since~$t$ is a stable model of~$\Gamma_D$ this implies that~$h' = t$.
However, $\tuple{h,t} \models \Gamma_I$ hand~$h \subset t$ also imply that~$\tuple{h,t} \not\models \Gamma_D$ and, thus, that~$\tuple{h,t} \not\models c \to F_c$ for some~$c \in C$.
This implies $\tuple{h,t} \models c$ and~$\tuple{h,t} \not\models F_c$.
Therefore, $h'(x_c) = \undefined$ and~$t(x_c) \neq\undefined$, which is a contradiction with the fact that~$h = t$.\end{proof}

\begin{lemma}\label{lem:aux1:prop:translations}
Let~$P$ be a $\TheoryAtoms$-program over~$\langle\Atoms,\TheoryAtoms,\TheoryAtomsES\rangle$
such that no atom in~$\TheoryAtomsES$ occur in a head of~$P$.
Let $\htcsemanticsP_3(P,\AT,\TheoryAtomsES)$ be the theory
\begin{align*}
    &\quad\ \htctsols(\AT,\TheoryAtomsES) \cup p'(P)\\
    &\cup\{ \htcsemanticsA(\thAt)\leftrightarrow p(\thAt) \mid \thAt \in \TheoryAtomsES\}\\
    &\cup\{ p(\thAt) \to \htcsemanticsA(\thAt) \mid \thAt \in \TheoryAtomsDN\}
    \quad\text{ for }\TheoryAtomsDN=\TheoryAtoms\setminus\TheoryAtomsES.
\end{align*}
where~$p'(P)$
is obtained from~$p(P)$ by replacing each occurrence of~$p(\thAt)$ with~$\thAt \in \TheoryAtomsES$ by~$\htcsemanticsA(\thAt)$.
Then, theories~$\htcsemanticsP_2(P,\AT,\TheoryAtomsES)$
and~$\htcsemanticsP_3(P,\AT,\TheoryAtomsES)$ have the same stable models.
\end{lemma}

\begin{proof}
Note  that by assumption no atom in~$\TheoryAtomsES$ occur in a head in~$P$.
By construction, this implies that no atom in~$p(\TheoryAtomsES)$ occurs in a head of~$p(P)$ and, thus, no variable of the form~$p(\thAt)$ with~$\thAt \in \TheoryAtomsES$ occurs in a head~$p(P)$ nor~$\htctsols(\AT,\TheoryAtomsES)$.
Hence, from Lemma~\ref{lem:completion.of.definitions}, we get that~$\htcsemanticsP_2(P,\AT,\TheoryAtomsES)$ and
\begin{align*}
    &\quad\ \htctsols(\AT,\TheoryAtomsES) \cup p(P)\\
    &\cup \{\phantom{\neg}\htcsemanticsA(\thAt)\leftrightarrow p(\thAt) \mid \thAt \in \TheoryAtomsES\}\\
    &\cup \{\neg \htcsemanticsA(\thAt) \wedge p(\thAt) \to \bot \mid \thAt \in \TheoryAtomsDN\}
    \text{ for }\TheoryAtomsDN=\TheoryAtoms\setminus\TheoryAtomsES.
\end{align*}
have the same stable models.
Second, since~$\htctsols(\AT,\TheoryAtomsES)$ contains formula~$\tau(\thAt) \vee \neg \tau(\thAt)$ for all~$\thAt \in \TheoryAtoms$, we can replace~$\neg \htcsemanticsA(\thAt) \wedge p(\thAt) \to \bot$ by~$p(\thAt) \to \htcsemanticsA(\thAt)$ without changing the stable models.
Finally, the statement of the lemma follows by application of the rule of substitution of equivalents.
\end{proof}

\begin{lemma}\label{lem:consistent:aux1b:prop:translations}
    Let~$\AT$ be a compositional and consistent theory and~$\tuple{h,t}$ be an interpretation.
Then, $\tuple{h,t} \models \tau(\comp{\thAt})$
    implies~$\tuple{h,t} \models \neg\tau(\thAt)$
    for any theory atom~$\thAt$.
\end{lemma}

\begin{proof}
Assume that~$\tuple{h,t} \models \tau(\comp{\thAt})$  and suppose, for the sake of contradiction, that~$\tuple{h,t} \not\models \neg\tau(\thAt)$.
From Proposition~1 by~\citeN{cafascwa20a}, the latter implies~$t \models \tau(\thAt)$.
That is, $t \in \den{ \tau(\thAt)}$.
Similarly, $\tuple{h,t} \models \tau(\comp{\thAt})$ implies~$t \in \den{ \tau(\comp{\thAt})}$.
Since~$\AT$ is compositional, this implies that~$\{ \thAt, \comp{\thAt} \}$ is \mbox{$\AT$-satisfiable}, which is a contradiction with the fact that~$\AT$ is compositional.
Consequently, $\tuple{h,t} \models \neg\tau(\thAt)$.
\end{proof}

\begin{lemma}\label{lem:aux1b:prop:translations}
    Let~$P$ be a $\TheoryAtoms$-program over~$\langle\Atoms,\TheoryAtoms,\TheoryAtomsES\rangle$ such that no atom in~$\TheoryAtomsES$ occurs in a head of~$P$.
Let~$\AT$ be a compositional and consistent theory.
Let $\htcsemanticsP_4(P,\AT,\TheoryAtomsES)$ be the theory
    \begin{align*}
        &\quad\ p'(P) \cup \tau(P)\\
        &\cup\{ \htcsemanticsA(\thAt)\vee \neg\htcsemanticsA(\thAt) \mid \thAt \in \TheoryAtomsDN \}\\
        &\cup\{ \htcsemanticsA(\thAt)\vee \htcsemanticsA(\comp{\thAt}) \ \ \mid \thAt \in \TheoryAtomsES\}\\
        &\cup\{ \htcsemanticsA(\thAt)\leftrightarrow p(\thAt) \ \mid \thAt \in \TheoryAtomsES\}\\
        &\cup\{ p(\thAt) \to \htcsemanticsA(\thAt) \ \mid \thAt \in \TheoryAtomsDN\}
        \quad\text{ for }\TheoryAtomsDN=\TheoryAtoms\setminus\TheoryAtomsES.
    \end{align*}
    where~$p'(P)$
    is obtained from~$p(P)$ by replacing each occurrence of~$p(\thAt)$ with~$\thAt \in \TheoryAtomsES$ by~$\htcsemanticsA(\thAt)$.
Then, theories~$\htcsemanticsP_2(P,\AT,\TheoryAtomsES)$
    and~$\htcsemanticsP_4(P,\AT,\TheoryAtomsES)$ have the same stable models.
\end{lemma}

\begin{proof}
From Lemma~\ref{lem:aux1:prop:translations}, theories~$\htcsemanticsP_2(P,\AT,\TheoryAtomsES)$
and~$\htcsemanticsP_3(P,\AT,\TheoryAtomsES)$ have the same stable models.
Let $\htcsemanticsP_4'(P,\AT,\TheoryAtomsES)$ be the theory
\begin{align*}
    &\quad\ p'(P) \cup \tau(P)\\
    &\cup\{ \htcsemanticsA(\thAt)\vee \htcsemanticsA(\comp{\thAt}) \ \mid \thAt \in \TheoryAtomsES\}\\
    &\cup\{ \htcsemanticsA(\thAt)\leftrightarrow p(\thAt) \mid \thAt \in \TheoryAtomsES\}\\
    &\cup\{ p(\thAt) \to \htcsemanticsA(\thAt) \mid \thAt \in \TheoryAtomsDN\}
    \quad\text{ for }\TheoryAtomsDN=\TheoryAtoms\setminus\TheoryAtomsES.
\end{align*}
Then,
$\htcsemanticsP_3(P,\AT,\TheoryAtomsES) \subseteq \htcsemanticsP_4'(P,\AT,\TheoryAtomsES)$ and, thus, every here-and-there model of~$\htcsemanticsP_4'(P,\AT,\TheoryAtomsES)$ is also a model of~$\htcsemanticsP_3(P,\AT,\TheoryAtomsES)$.
Note also that every formula in~${\htcsemanticsP_4'(P,\AT,\TheoryAtomsES) \setminus \htcsemanticsP_3(P,\AT,\TheoryAtomsES)}$
is either of the form of~$B \to \tau(\thAt)$ or~$\htcsemanticsA(\thAt)\vee \htcsemanticsA(\comp{\thAt})$.
For each formula~$B \to \tau(\thAt)$, there are formulas~$B \to p(\thAt)$ and~$p(\thAt) \to \htcsemanticsA(\thAt)$ in~$\htcsemanticsP_3(P,\AT,\TheoryAtomsES)$, which intuitionistically entail it.
Similarly, for each formula~$\htcsemanticsA(\thAt)\vee \htcsemanticsA(\comp{\thAt})$, there are formulas~\eqref{eq:guess:choice} and~\eqref{eq:guess:strict}, which intuitionistically entail it.
Hence, $\htcsemanticsP_3(P,\AT,\TheoryAtomsES)$ and~$\htcsemanticsP_4'(P,\AT,\TheoryAtomsES)$ have the same here-and-there models.
Let~$\htcsemanticsP_4''(P,\AT,\TheoryAtomsES)$ be the theory
\begin{align*}
    &\quad\ p'(P) \cup \tau(P)\\
    &\cup\{ \htcsemanticsA(\thAt)\vee \neg\htcsemanticsA(\thAt) \mid \thAt \in \TheoryAtoms \}\\
    &\cup\{ \htcsemanticsA(\thAt)\vee \htcsemanticsA(\comp{\thAt}) \ \mid \thAt \in \TheoryAtomsES\}\\
    &\cup\{ \htcsemanticsA(\thAt)\leftrightarrow p(\thAt) \mid \thAt \in \TheoryAtomsES\}\\
    &\cup\{ p(\thAt) \to \htcsemanticsA(\thAt) \mid \thAt \in \TheoryAtomsDN\}
    \quad\text{ for }\TheoryAtomsDN=\TheoryAtoms\setminus\TheoryAtomsES.
\end{align*}
Then, $\htcsemanticsP_4'(P,\AT,\TheoryAtomsES)$ is the result of adding
\begin{align}
    \neg \htcsemanticsA(\thAt) &\to \, \htcsemanticsA(\comp{\thAt}) &&\text{for every theory atom } \thAt \in \TheoryAtomsES
\end{align}
to~${\htcsemanticsP_4''(P,\AT,\TheoryAtomsES)}$ and that~${\neg \htcsemanticsA(\thAt) \to \htcsemanticsA(\comp{\thAt})}$ is entailed by~$\htcsemanticsA(\thAt)\vee \htcsemanticsA(\comp{\thAt})$.
Hence, theories~${\htcsemanticsP_4'(P,\AT,\TheoryAtomsES)}$ and~${\htcsemanticsP_4''(P,\AT,\TheoryAtomsES)}$ have the same here-and-there models.
Finally, note that ${\htcsemanticsP_4''(P,\AT,\TheoryAtomsES)}$ is the result of adding
\begin{align}
    \htcsemanticsA(\thAt)\vee \neg\htcsemanticsA(\thAt) &&\text{for every theory atom } \thAt \in \TheoryAtomsES
\end{align}
to~${\htcsemanticsP_4(P,\AT,\TheoryAtomsES)}$ and that, from Lemma~\ref{lem:consistent:aux1b:prop:translations}, we get that~${\htcsemanticsA(\thAt)\vee \htcsemanticsA(\comp{\thAt})}$ entails~${\htcsemanticsA(\thAt)\vee \neg\htcsemanticsA(\thAt)}$.
Hence, ${\htcsemanticsP_4''(P,\AT,\TheoryAtomsES)}$ and~${\htcsemanticsP_4(P,\AT,\TheoryAtomsES)}$ have the same here-and-there models and the lemma statement follows.\end{proof}

\begin{lemma}\label{lem:aux2a:prop:translations}
    Let~$P$ be a $\TheoryAtoms$-program over~$\langle\Atoms,\TheoryAtoms,\TheoryAtomsES\rangle$.
Let~$\htcsemanticsP_4(P,\AT,\TheoryAtomsES)$ a theory as defined in Lemma~\ref{lem:aux1b:prop:translations}.
Let~$v$ and~$w$ be valuations such that
    \begin{align*}
        w &= \, \{ \ (\tau(\regAt),\mathbf{t}) \, \mid \regAt \in \Atoms \text{ and } v(\tau(\regAt)) = \mathbf{t}  \ \}
        \\
        &\ \cup \{ \ (x,v(x)) \mid x \in \varsAT{\thAt}, \ \thAt \in \TheoryAtomsES \text{ and } v(x) \neq \undefined\ \}
        \\
        &\ \cup \{ \ (x,v(x)) \mid x \in \varsAT{\thAt}, \ \thAt \in \TheoryAtomsDN, \ v \in \den{ p(\thAt) } \text{ and } v(x) \neq \undefined  \ \}
    \end{align*}
If $v \models \htcsemanticsP_4(P,\AT,\TheoryAtomsES)$,
then $w \models \htcsemanticsA(P,\AT,\TheoryAtomsES)$.
\end{lemma}

\begin{proof}
    Note that~$\htcsemanticsP_4(P,\AT,\TheoryAtomsES)$ can be rewritten as
    \begin{align*}
        &\quad\ \htcsemanticsP(P,\AT,\TheoryAtomsES) \cup p'(P)\\
        &\cup\{ \htcsemanticsA(\thAt)\vee \neg\htcsemanticsA(\thAt) \mid \thAt \in \TheoryAtomsDN \}\\
        &\cup\{ \htcsemanticsA(\thAt)\leftrightarrow p(\thAt) \ \mid \thAt \in \TheoryAtomsES\}\\
        &\cup\{ p(\thAt) \to \htcsemanticsA(\thAt) \ \mid \thAt \in \TheoryAtomsDN\}
        \quad\text{ for }\TheoryAtomsDN=\TheoryAtoms\setminus\TheoryAtomsES.
    \end{align*}
    and recall that~$\htcsemanticsP(P,\AT,\TheoryAtomsES)$ is
    \begin{gather*}
        \ \tau(P) \cup\{  \htcsemanticsA(\thAt) \lor \htcsemanticsA(\comp{\thAt}) \mid \thAt \in \TheoryAtomsES \}
    \end{gather*}
    Assume that $v \models \htcsemanticsP_4(P,\AT,\TheoryAtomsES)$.
Clearly~${w \models \htcsemanticsA(\thAt) \lor \htcsemanticsA(\comp{\thAt})}$ for all~${\thAt\in \TheoryAtomsES}$ because~$v$ and~$w$ agree on all variables occurring in atoms from~$\TheoryAtomsES$ and this formula belong to~$\htcsemanticsP(P,\AT,\TheoryAtomsES)$.
Hence, it remains to be shown that~$w \models \tau(P)$.
Since~$P$ is regular, all formulas in~$\tau(P)$ are of the form~$B \to \tau(\thAt)$ with~$\thAt \in \Atoms \cup \TheoryAtomsDN$.
Pick any such a formula.
If~${w \models B}$, then~${v \models }B$ because all atoms in~$B$ belong to~$\Atoms \cup \TheoryAtomsES$ and~$v$ and~$w$ agree on all variables occurring in these atoms.
We distinguish two cases:
    \begin{itemize}
        \item If~${\thAt \in \Atoms}$, then ${B \to \tau(\thAt)}$ in~$\tau(P)$ implies that~${B \to \tau(\thAt)}$ belongs to~$p'(P)$
        and, since~$v \models \htcsemanticsP_4(P,\AT,\TheoryAtomsES)$, this implies $w( \tau(\thAt)) = v( \tau(\thAt)) = \mathbf{t}$.

        \item If~${\thAt \in \TheoryAtomsDN}$, then~${B \to \tau(\thAt)}$ in~$\tau(P)$ implies that~${B \to p(\thAt)}$ belongs to~$p'(P)$.
This implies that~${v \in \den{  p(\thAt) }}$ and, thus, $v$ and~$w$ agree on all variables occurring in~$\thAt$.
Hence, $w \in \den{  p(\thAt) }$.
    \end{itemize}
    Therefore, we get that~$w \models \htcsemanticsP(P,\AT,\TheoryAtomsES)$.
\end{proof}

\begin{lemma}\label{lem:aux2b:prop:translations}
    Let~$P$ be a $\TheoryAtoms$-program over~$\langle\Atoms,\TheoryAtoms,\TheoryAtomsES\rangle$.
Let~$v$ and~$w$ be valuations such that
    \begin{align*}
        v = w &\cup \{ \ (p(\thAt),\mathbf{t})\mid \thAt \in \TheoryAtomsES \text{ and } w \in \den{ \tau(\thAt) } \ \}
        \\
        &\cup \{ \ (p(\thAt),\mathbf{t})\mid \thAt \in \TheoryAtomsDN, \ (B \to p(\thAt))) \in \tau(P) \text{ and } w \models B \ \}
    \end{align*}
If $v$ is a stable model of theory~$\htcsemanticsP_2(P,\AT,\TheoryAtomsES)$,
    then $w$ is a stable model of~$\htcsemanticsA(P,\AT,\TheoryAtomsES)$.
    \end{lemma}

    \begin{proof}
    Assume that $v$ is a stable model of theory~$\htcsemanticsP_2(P,\AT,\TheoryAtomsES)$.
From Lemma~\ref{lem:aux1b:prop:translations}, this implies that~$v$ is a stable model of theory~$\htcsemanticsP_4(P,\AT,\TheoryAtomsES)$ and, thus, that~$v \models \htcsemanticsP_4(P,\AT,\TheoryAtomsES)$.
From Lemma~\ref{lem:aux2a:prop:translations}, this implies that~$w \models \htcsemanticsP(P,\AT,\TheoryAtomsES)$.
Let as show that~$w$ is a stable model of this theory.
Pick any valuation~$h \subset w$
    and let~
    \begin{align*}
        h'  \ \ = \ \ h &\cup \{ \ (p(\thAt),\mathbf{t}) \,\mid
            \thAt \in \TheoryAtomsES, \
            h \in \den{ \tau(\thAt)}
            \text{ and } v(p(\thAt)) = \mathbf{t}   \ \}
        \\
        &\cup \{ \ (p(\thAt),\mathbf{t}) \,\mid
        \thAt \in \TheoryAtomsDN, \
            (B \to \tau(\thAt)) \in \tau(P)
            \text{ and } \tuple{h,w} \models B
        \ \}
        \\
        &\cup \{ \ (x,v(x)) \mid
        \thAt \in \TheoryAtomsDN, \
            x \in \varsAT{\thAt}
            \text{ and }
            v \in \den{\tau(\thAt) }
        \ \}
    \end{align*}
    Let us show first that~${h' \subseteq v}$.
The only non-trivial case is when~${h'(p(\thAt)) = \mathbf{t}}$ and~${\thAt \in \TheoryAtomsDN}$.
This implies there is~$B \to \tau(\thAt)$ in~$\tau(P)$ such that~${ \tuple{h,w} \models B}$
    and, thus, that~${ w \models B}$.
Since~$w$ agrees with~$v$ on all variables occurring in~$B$, it follows that~${v \models B}$.
In addition, since~$v \models \htcsemanticsP_4(P,\AT,\TheoryAtomsES)$,
    we get that ${v \models B \to p(\thAt)}$ and, thus, ${v(p(\thAt)) = \mathbf{t}}$.
Note that~$B \to \tau(\thAt)$ in~$\tau(P)$ implies that~$B \to p(\thAt)$ belongs to~$p'(P)$.
Hence, ${h' \subseteq v}$.
    \\[10pt]
    Then, either~${h'= v}$ or~${h' \subset v}$.
We proceed by cases and we will prove~${\tuple{h,w} \not\models \htcsemanticsP(P,\AT,\TheoryAtomsES)}$.
    \begin{itemize}
        \item Assume that~${h'= v}$.
Since~${h \subset w}$, there is~${\thAt \in \TheoryAtomsDN}$
        and
        ${x \in \varsAT{\thAt}}$
        such that~${h(x) = \undefined}$,
        ${w(x)  = v(x) \neq \undefined}$,
        ${v \in \den{\tau(\thAt)}}$
        and~${v \in \den{ p(\thAt) }}$.
Furthermore, since~$v$ is a stable model of~$\htcsemanticsP_4(P,\AT,\TheoryAtomsES)$,
        we get that~${v \in \den{ p(\thAt) }}$ implies that there is a formula of the form~${B \to p(\thAt)}$ in~$p'(P)$ such that~${v \models B}$ (Proposition~11 by~\citeNP{cafascwa20b}).
Since~$v$ and~$w$ agree on all variable occurring in atoms from~$\Atoms \cup \TheoryAtomsES$, this implies that~${w \models B}$.
In addition, $\tau(P)$ contains formula~${B \to \tau(\thAt)}$ and~${\tuple{h,w} \not\models \tau(\thAt) }$ because ${x \in \varsAT{\thAt}}$ and~${h(x) = \undefined}$.
Suppose, for the sake of contradiction, that~${\tuple{h,w} \models \htcsemanticsP(P,\AT,\TheoryAtomsES)}$.
Then, ${\tuple{h,w} \not\models B }$ and,
        since~${w \models B}$, there is an atom~${\anyAt \in \Atoms \cup \TheoryAtomsES}$ that occurs in~$B$ such that~${ \tuple{h,w} \not\models \tau(\anyAt) }$.
This implies that there is variable~${y \in \vars{\anyAt}_\AT}$ such that~${h(y) = \undefined}$
        and~${w(y) = v(y) \neq \undefined}$.
Note that, if~${\anyAt \in \Atoms}$, then~$y =p_\anyAt$ and, thus, ~$h(p_\anyAt) = h'(p_\anyAt) = w(p_\anyAt)$, which is a contradiction.
Hence, we get that~${\anyAt \in \TheoryAtomsES}$ and, since~${ \vars{\anyAt}_\AT = \vars{\comp{\anyAt}}_\AT}$, this implies that
        ${\tuple{h,w} \not\models \tau(\comp{\anyAt}) }$.
This is a contradiction with the assumption, because~${\htcsemanticsP(P,\AT,\TheoryAtomsES)}$ contains a formula of the form~$\tau(\anyAt) \vee \tau(\comp{\anyAt})$ for every~$\anyAt \in \TheoryAtomsES$.
Consequently, ${\tuple{h,w} \not\models \htcsemanticsP(P,\AT,\TheoryAtomsES)}$.

        \item Assume now that~${h'\subset v}$.
Since~$v \models \htcsemanticsP_4(P,\AT,\TheoryAtomsES)$, it follows that
        \begin{align*}
v &\models \htcsemanticsA(\thAt)\leftrightarrow p(\thAt) &&\text{ for all~$\thAt \in \TheoryAtomsES$}
\end{align*}
        Then, it is easy to see that the following holds by construction:
        \begin{align}
            \tuple{h',v} &\models \htcsemanticsA(\thAt)\vee \neg\htcsemanticsA(\thAt) &&\text{ for all~$\thAt \in \TheoryAtomsDN$}
            \label{eq:1aa:lem:aux2:prop:translations}
            \\
            \tuple{h',v} &\models \htcsemanticsA(\thAt)\leftrightarrow p(\thAt) &&\text{ for all~$\thAt \in \TheoryAtomsES$}
            \label{eq:1bb:lem:aux2:prop:translations}
\end{align}
        Suppose, for the sake of contradiction, that
        \begin{gather}
            \tuple{h,w} \models \htcsemanticsP(P,\AT,\TheoryAtomsES)
            \label{eq:2:lem:aux2:prop:translations}
        \end{gather}
        We will show that this implies~$\tuple{h',v} \models \htcsemanticsP_4(P,\AT,\TheoryAtomsES)$, which is a contradiction with the fact that~$v$ is a stable model of this theory.
Pick a formula of the form~$p(\thAt) \to \htcsemanticsA(\thAt)$ with~$\thAt \in \TheoryAtomsDN$.
If~${\tuple{h',v} \models p(\thAt)}$,
        by construction there is a formula of the form of ${B \to \tau(\thAt)}$ in~$\tau(P)$ such that~${\tuple{h,w} \models B}$.
This implies~${\tuple{h',v} \models B}$.
Recall that~$w$ and~$v$ agree on all variables occurring in atoms from~$\Atoms \cup \TheoryAtomsES$ and~${h \subseteq h'}$.
Since we supposed that~$\tuple{h',v} \models \htcsemanticsP(P,\AT,\TheoryAtomsES)$ and this theory contains~${B \to \tau(\thAt)}$,
        this implies that~${\tuple{h',v} \models \tau(\thAt)}$.
Since~${\tuple{h',v} \models p(\thAt)}$, we get that~$v \in \den{ p(\thAt)}$ and, by construction, these two facts imply that~${\tuple{h,w} \models \tau(\thAt)}$.
Hence,
        \begin{align}
            \tuple{h',v} &\models  p(\thAt) \to \htcsemanticsA(\thAt) &&\text{ for all~$\thAt \in \TheoryAtomsDN$}
            \label{eq:1cc:lem:aux2:prop:translations}
        \end{align}
        Pick now a formulas of the form~${B \to p(\thAt)}$ and assume that~${\tuple{h',v} \models  B}$.
Then, there is
        ${B \to \tau(\thAt)}$ in~$\tau(P)$
        and~${\tuple{h,w} \models B}$.
By construction, this implies that~$h'(p(\thAt)) = \mathbf{t}$.
Hence,
        \begin{align}
            \tuple{h',v} &\models  B \to p(\thAt) &&\text{ for all~$B \to p(\thAt)$ in~$p'(P)$}
            \label{eq:1dd:lem:aux2:prop:translations}
        \end{align}
        Taking together facts~(\ref{eq:1aa:lem:aux2:prop:translations}-\ref{eq:1dd:lem:aux2:prop:translations}),
        we get that~${\tuple{h',v} \models \htcsemanticsP_4(P,\AT,\TheoryAtomsES)}$, which is a contradiction with the fact that~$v$ is a stable model of~$\htcsemanticsP_4(P,\AT,\TheoryAtomsES)$.
Hence,~${\tuple{h,w} \not\models \htcsemanticsP(P,\AT,\TheoryAtomsES)}$.
    \end{itemize}
    In both cases, we get that~${\tuple{h,w} \not\models \htcsemanticsP(P,\AT,\TheoryAtomsES)}$ and, thus, $w$ is a stable model of~${\htcsemanticsP(P,\AT,\TheoryAtomsES)}$.\end{proof}

    \begin{lemma}
        \label{lem:aux3:prop:translations}
        Let~$P$ be a $\TheoryAtoms$-program over~$\langle\Atoms,\TheoryAtoms,\TheoryAtomsES\rangle$ such that
        all atoms occurring in a body of~$P$ belong to~$\TheoryAtomsES$ and
        no atom in~$\TheoryAtomsES$ occur in a head of~$P$.
Let~$v$ and~$w$ be valuations such that
        \begin{align*}
            v = w &\cup \{ \ (p(\thAt),\mathbf{t})\mid \thAt \in \TheoryAtomsES \text{ and } w \in \den{ \tau(\thAt) } \ \}
            \\
            &\cup \{ \ (p(\thAt),\mathbf{t})\mid \thAt \in \TheoryAtomsDN, \ (B \to p(\thAt))) \in \tau(P) \text{ and } w \models B \ \}
        \end{align*}
        If $w$ is a stable model of theory~$\htcsemanticsA(P,\AT,\TheoryAtomsES)$,
        then $v$ is a stable model of~$\htcsemanticsP_2(P,\AT,\TheoryAtomsES)$.
        \end{lemma}
    \begin{proof}

    Assume that $w$ is a stable model of theory~$\htcsemanticsP(P,\AT,\TheoryAtomsES)$.
We will show that~$v$ is a stable model of~$\htcsemanticsP_4(P,\AT,\TheoryAtomsES)$, which from Lemma~\ref{lem:aux1b:prop:translations}, implies that~$v$ is a stable model of~$\htcsemanticsP_4(P,\AT,\TheoryAtomsES)$.
Let us start by showing that~$v \models \htcsemanticsP_4(P,\AT,\TheoryAtomsES)$.
Recall that~$\htcsemanticsP_4(P,\AT,\TheoryAtomsES)$ can be rewritten as
    \begin{align}
        &\quad\ \htcsemanticsP(P,\AT,\TheoryAtomsES)
            \label{eq:1a:lem:aux3:prop:translations}\\
        &\cup p'(P)
            \label{eq:1b:lem:aux3:prop:translations}\\
        &\cup\{ \htcsemanticsA(\thAt)\vee \neg\htcsemanticsA(\thAt) \, \mid \thAt \in \TheoryAtomsDN \}
            \label{eq:1c:lem:aux3:prop:translations}\\
        &\cup\{ \htcsemanticsA(\thAt)\leftrightarrow p(\thAt) \ \mid \thAt \in \TheoryAtomsES\}
            \label{eq:1d:lem:aux3:prop:translations}\\
        &\cup\{ p(\thAt) \to \htcsemanticsA(\thAt) \ \mid \thAt \in \TheoryAtomsDN\}
        \quad\text{ for }\TheoryAtomsDN=\TheoryAtoms\setminus\TheoryAtomsES.
            \label{eq:1e:lem:aux3:prop:translations}
    \end{align}
    Note that, since~$v$ and~$w$ agree on all variable occurring in~$\htcsemanticsP(P,\AT,\TheoryAtomsES)$, we immediately get that~$v$ satisfies all formulas in it.
Furthermore, it is clear~$v$ satisfies all formulas in~\eqref{eq:1c:lem:aux3:prop:translations} and, by construction, all formulas in~\eqref{eq:1d:lem:aux3:prop:translations}.
Pick a formula of the form of~$p(\thAt) \to \htcsemanticsA(\thAt)$ and assume that~$v \models p(\thAt)$.
By construction, this means that there is a formula of the form~$B \to \tau(\thAt)$ in~$\tau(P) \subseteq \htcsemanticsP(P,\AT,\TheoryAtomsES)$ such that~$w \models B$ and, thus, $w \in \den{ \htcsemanticsP(\thAt)}$.
Since~$v$ and~$w$ agree on all variable occurring in theory atoms, we immediately get that~$v \models \htcsemanticsP(\thAt)$ and, thus, all formulas in~\eqref{eq:1e:lem:aux3:prop:translations}.
Finally, pick any formula in~$p'(P) \setminus \htcsemanticsP(P,\AT,\TheoryAtomsES)$.
This formula is of the form~$B \to p(\thAt)$ with~$\thAt \in \TheoryAtomsDN$.
Furthermore, if~$v \models B$, we get that~$w \models B$ and, by construction, this implies that~$w \models p(\thAt)$.
Consequently, $v$ satisfies all formulas in~\eqref{eq:1b:lem:aux3:prop:translations} and, thus, we get that
    $v \models \htcsemanticsP_4(P,\AT,\TheoryAtomsES)$.
    \\[10pt]
    Let us show now that, in fact, $v$ is a stable model of~$\htcsemanticsP_4(P,\AT,\TheoryAtomsES)$.
Pick any valuation~${h \subset v}$ and suppose, for the sake of contradiction, that~${{\tuple{h,v}} \models {\htcsemanticsP_4(P,\AT,\TheoryAtomsES)}}$.
Then, $\restr{h}{\X} \subseteq w$ and, since~$w$ is a stable model of theory~$\htcsemanticsP(P,\AT,\TheoryAtomsES)$, we get that either~${\restr{h}{\X} = w}$ or ${\tuple{\restr{h}{\X},w} \not\models \htcsemanticsP(P,\AT,\TheoryAtomsES)}$.
    \begin{itemize}
        \item Assume ${\restr{h}{\X} = w}$. Since ${h \subset v}$, this implies that ${h(p(\thAt)) = \undefined}$ and ${v(p(\thAt)) = \mathbf{t}}$ for some~${\thAt \in \TheoryAtoms}$.
        \begin{itemize}
            \item Assume first that~${\thAt \in \TheoryAtomsES}$.
Then, by construction, ${v(p(\thAt)) = \mathbf{t}}$ implies~${w \in \den{\tau(\thAt)}}$
            and, since~${\vars{\tau(\thAt)} \subseteq \X}$, this plus~${\restr{h}{\X} = w}$
            imply~${h \in \den{\tau(\thAt)}}$.
Furthermore, since we supposed that~${\tuple{h,v} \models \htcsemanticsP_4(P,\AT,\TheoryAtomsES)}$,
            it follows that~${\tuple{h,v} \models \htcsemanticsA(\thAt)\leftrightarrow p(\thAt)}$ for all~$\thAt \in \TheoryAtomsES$, which is a contradiction with facts~${h(p(\thAt)) = \undefined}$ and~${h \in \den{\tau(\thAt)}}$.

            \item Assume now that~${\thAt \in \TheoryAtomsDN}$.
By construction, ${v(p(\thAt)) = \mathbf{t}}$ implies that there is a formula of the form of~${B \to \tau(\thAt)}$ in~$\tau(P)$ such that~${w \models B}$.
Since we assumed~${\restr{h}{\X} = w}$, this implies that~$\tuple{h,v} \models B$.
Recall that~$v$ and~$w$ agree on all variables occurring in theory atoms.
Furthermore, since~${\thAt \in \TheoryAtomsDN}$, we also get that~${B \to p(\thAt)}$ belongs to~${p'(P) \subseteq \htcsemanticsP_4(P,\AT,\TheoryAtomsES)}$.
Then, since~${\tuple{h,v}}$ satisfies~${\htcsemanticsP_4(P,\AT,\TheoryAtomsES)}$, these two facts imply that~$\tuple{h,v} \models p(\thAt)$, which is a contradiction with~${h(p(\thAt)) = \undefined}$.
        \end{itemize}

        \item Assume now~${\tuple{\restr{h}{\X},w} \not\models \htcsemanticsP(P,\AT,\TheoryAtomsES)}$.
This implies that~${\tuple{h,v} \not\models \htcsemanticsP(P,\AT,\TheoryAtomsES)}$
        because interpretation~$\tuple{h,v}$ agrees with~${\tuple{\restr{h}{\X},w}}$
        on all variables occurring in~$\htcsemanticsP(P,\AT,\TheoryAtomsES)$.
Since~$\htcsemanticsP(P,\AT,\TheoryAtomsES) \subseteq \htcsemanticsP_4(P,\AT,\TheoryAtomsES)$,
        this implies that~${\tuple{h,v} \not\models \htcsemanticsP_4(P,\AT,\TheoryAtomsES)}$, which is a contradiction with the supposition.
    \end{itemize}
    Hence,~${\tuple{h,v} \not\models \htcsemanticsP_4(P,\AT,\TheoryAtomsES)}$ and, thus, that~$v$ is a stable model of~$\htcsemanticsP_4(P,\AT,\TheoryAtomsES)$.
From Lemma~\ref{lem:aux1b:prop:translations}, this implies that~$v$ is a stable model of~$\htcsemanticsP_2(P,\AT,\TheoryAtomsES)$.
    \end{proof}

\begin{proof}[Proof of Proposition~\ref{prop:translations}]
    The if direction follows directly from Lemma~\ref{lem:aux2b:prop:translations} and the only if direction from Lemma~\ref{lem:aux3:prop:translations}.
\end{proof}

\end{document}